\documentclass{article}

\usepackage{microtype}
\usepackage{graphicx}
\usepackage{subfigure}
\usepackage{booktabs} 

\usepackage{hyperref}

\usepackage[accepted]{icml2019}

\usepackage{amsmath,amsfonts,bm}

\def\Thmref#1{Theorem~\ref{#1}}
\def\Lemref#1{Lemma~\ref{#1}}

\def\eqref#1{equation~\ref{#1}}

\def\1{\bm{1}}

\newcommand{\reals}{\mathbb{R}}
\newcommand{\lfuncs}{C_{L}(X,\mathbb{R})}
\newcommand{\leftpar}{\left|\left|}
\newcommand{\rightpar}{\right|\right|}

\def\rx{{\textnormal{x}}}

\def\rvb{{\mathbf{b}}}

\def\rvf{{\mathbf{f}}}
\def\rvg{{\mathbf{g}}}
\def\rvh{{\mathbf{h}}}
\def\rvu{{\mathbf{i}}}

\def\rvu{{\mathbf{u}}}
\def\rvv{{\mathbf{v}}}

\def\rvx{{\mathbf{x}}}
\def\rvy{{\mathbf{y}}}
\def\rvz{{\mathbf{z}}}

\def\rmW{{\mathbf{W}}}

\def\vb{{\bm{b}}}

\def\vh{{\bm{h}}}

\def\vn{{\bm{n}}}

\def\vu{{\bm{u}}}

\def\vx{{\bm{x}}}
\def\vy{{\bm{y}}}
\def\vz{{\bm{z}}}

\DeclareMathAlphabet{\mathsfit}{\encodingdefault}{\sfdefault}{m}{sl}
\SetMathAlphabet{\mathsfit}{bold}{\encodingdefault}{\sfdefault}{bx}{n}

\def\gX{{\mathcal{X}}}
\def\gY{{\mathcal{Y}}}

\newcommand{\R}{\mathbb{R}}

\usepackage{amsfonts, amsmath, amsthm}
\usepackage{hyperref}
\usepackage{url}
\usepackage{graphicx}
\usepackage{color}
\usepackage{times}
\usepackage{caption}
\usepackage{multicol, multirow, wrapfig}
\usepackage{placeins}
\usepackage{listings}
\usepackage[toc,page]{appendix}
\usepackage{thmtools,thm-restate}
\usepackage{amssymb}
\usepackage{mleftright}
\usepackage{xparse}
\usepackage{textcomp}
\usepackage{diagbox}

\newif\ifcomments

\commentsfalse

\ifcomments
  \newcommand{\colornote}[3]{{\color{#1}\bf{#2: #3}\normalfont}}
\else
  \newcommand{\colornote}[3]{}
\fi

\NewDocumentCommand{\evalat}{sO{\big}mm}{%
  \IfBooleanTF{#1}
  {\mleft. #3 \mright|_{#4}}
  {#3#2|_{#4}}%
}

\newtheorem{definition}{Definition}
\newtheorem*{remark}{Remark}

\icmltitlerunning{Sorting Out Lipschitz Function Approximation}

\begin{document}

\twocolumn[
\icmltitle{Sorting Out Lipschitz Function Approximation}



\icmlsetsymbol{equal}{*}

\begin{icmlauthorlist}
\icmlauthor{Cem Anil}{equal,to,ve}
\icmlauthor{James Lucas}{equal,to,ve}
\icmlauthor{Roger Grosse}{to,ve}
\end{icmlauthorlist}

\icmlaffiliation{to}{Department of Computer Science, University of Toronto, Toronto, Canada}
\icmlaffiliation{ve}{Vector Institute, Toronto, Canada}

\icmlcorrespondingauthor{Cem Anil}{cem.anil@mail.utoronto.ca}
\icmlcorrespondingauthor{James Lucas}{jlucas@cs.toronto.edu}

\icmlkeywords{Machine Learning, ICML}

\vskip 0.3in
]



\printAffiliationsAndNotice{\icmlEqualContribution} 

\begin{abstract}
Training neural networks under a strict Lipschitz constraint is useful for provable adversarial robustness, generalization bounds, interpretable gradients, and Wasserstein distance estimation. By the composition property of Lipschitz functions, it suffices to ensure that each individual affine transformation or nonlinear activation  is 1-Lipschitz. The challenge is to do this while maintaining the expressive power. We identify a necessary property for such an architecture: each of the layers must preserve the gradient norm during backpropagation. Based on this, we propose to combine a gradient norm preserving activation function, GroupSort, with norm-constrained weight matrices. We show that norm-constrained GroupSort architectures are universal Lipschitz function approximators. Empirically, we show that norm-constrained GroupSort networks achieve tighter estimates of Wasserstein distance than their ReLU counterparts and can achieve provable adversarial robustness guarantees with little cost to accuracy.
\end{abstract}

\setcounter{section}{1}
Constraining the Lipschitz constant of a neural network puts a bound on how much its output can change in proportion to a change in its input. 
For classification tasks, a small Lipschitz constant leads to better generalization \citep{sokolic2017robust}, improved adversarial robustness \citep{cisse2017parseval, tsuzuku2018lipschitz}, and greater interpretability \citep{tsipras2018there}. Additionally, the Wasserstein distance between two probability distributions can be expressed as the solution to a maximization problem over Lipschitz functions \citep{peyre2018computational}. Despite the wide-ranging applications, the question of how to approximate
Lipschitz functions with neural networks without sacrificing expressive power has remained largely unanswered.

Existing approaches to enforce Lipschitz constraints fall into two categories: regularization and architectural constraints. 
Regularization approaches \citep{drucker1992improving, gulrajani2017improved} perform well in practice, but do not provably enforce the Lipschitz constraint globally.
Approaches based on architectural constraints place limitations on the operator norm (such as the matrix spectral norm) of each layer's weights \citep{cisse2017parseval, yoshida2017spectral}. These provably satisfy the Lipschitz constraint, but come at a cost in expressive power. E.g., norm-constrained ReLU networks are provably unable to approximate simple functions such as absolute value \citep{huster2018limitations}. 

We first identify a simple property that expressive norm-constrained Lipschitz architectures must satisfy: gradient norm preservation. Specifically, in order to represent a function with slope 1 almost everywhere, each layer must preserve the norm of the gradient during backpropagation. ReLU architectures satisfy this only when the activations are positive; empirically, this manifests during training of norm-constrained ReLU networks in that the activations are forced to be positive most of the time, reducing the network's capacity to represent nonlinear functions. We make use of an alternative activation function called \emph{GroupSort} --- a variant of which was proposed by \citet{chernodub2016norm} --- which sorts groups of activations. GroupSort is both Lipschitz and gradient norm preserving. Using a variant of the Stone-Weierstrass theorem, we show that norm-constrained GroupSort networks are universal Lipschitz function approximators. While we focus our attention, both theoretically and empirically, on fully connected networks, the same general principles hold for convolutional networks where the techniques we introduce could be applied. 

Empirically, we show that ReLU networks are unable to approximate even the simplest Lipschitz functions which GroupSort networks can. We observe that norm-constrained ReLU networks must trade non-linear processing for gradient norm, leading to less expressive networks. Moreover, we obtain tighter lower bounds on the Wasserstein distance between complex, high dimensional distributions using GroupSort architectures. We also train classifiers with provable adversarial robustness guarantees and find that using GroupSort provides improved accuracy and robustness compared to ReLU. Across all of our experiments, we found that norm-constrained GroupSort architectures consistently outperformed their ReLU counterparts.

\section{Background}
\label{bckg}

\paragraph{Notation} We will use $\rvx \in \R^{in}$ to denote the input vector to the neural network, $\vy \in \R^{out}$  the output (or logits), $n_{l}$  the dimensionality of the $l^{th}$ hidden layer, $\rmW_{l} \in \R^{n_{l-1} \times n_{l}}$  and $\vb_{l} \in \R^{n_{l}}$  the weight matrix and the bias of the $l^{th}$ layer. We will denote the pre-activations in layer $l$ with $\vz_{l}$ and activations with $\vh_{l}$. The number of layers will be $L$ with $\vy=\vz_{L}$. We will use $\phi$ to denote the activation used. The computation performed by layer $l$ will be:
\[
\vz_{l} = \rmW_{l}\vh_{l-1} + \vb_{l} \quad \quad \vh_{l} = \phi(\vz_{l})
\]
\subsection{Lipschitz Functions}
Given two metric spaces $\gX$ and $\gY$, a function $f: \gX \rightarrow \gY$ is Lipschitz continuous if there exists $K \in \reals$ such that for all $x_1$ and $x_2$ in $\gX$,
\[ d_{\gY}(f(x_1), f(x_2)) \leq K d_{\gX}(x_1, x_2)\]
where $d_{\gX}$ and $d_{\gY}$ are metrics (such as Euclidean distance) on $\gX$ and $\gY$ respectively. In this work, when we refer to \emph{the} Lipschitz constant we are referring to the smallest such $K$ for which the above holds under a given $d_\gX$ and $d_\gY$. Unless otherwise specified, we take $\gX = \reals^n$ and $\gY=\reals^m$ throughout. If the Lipschitz constant of a function is $K$, it is called a \textit{$K$-Lipschitz} function. If the function is everywhere differentiable then its Lipschitz constant is bounded by the operator norm of its Jacobian. Throughout this work, we make use of the following definition:
\begin{definition}
Given a metric space $(X, d_{X})$ where $d_{X}$ denotes the metric on $X$, we write $\lfuncs$ to denote the space of all 1-Lipschitz functions mapping $X$ to $\reals$ (with respect to the $L_p$ metric).
\end{definition}
\subsection{Lipschitz-Constrained Neural Networks}

As 1-Lipschitz functions are closed under composition, to build a 1-Lipschitz neural network it suffices to compose 1-Lipschitz affine transformations and activations.

\paragraph{1-Lipschitz Linear Transformations: } Ensuring that each linear map is 1-Lipschitz is equivalent to ensuring that $||\rmW\rvx||_p \leq ||\rvx||_p$ for any $\rvx$; this is equivalent to constraining the matrix $p$-norm, $||\rmW||_p = \sup_{||\rvx||_p=1} ||\rmW\rvx||_p$, to be at most 1. Important examples of matrix $p$-norms include the matrix 2-norm, which is the largest singular value, and the matrix $\infty$-norm, which can be expressed as:
\[ ||\rmW||_\infty = \max_{1\leq i \leq m} \sum_{j=1}^{m} |w_{ij}|. \]
Similarly, we may also define the mixed matrix norm, given by $||\rmW||_{p,q} = \sup_{||\rvx||_p=1} ||\rmW\rvx||_{q}$. Enforcing matrix norm constraints naively may be computationally expensive. We discuss techniques to efficiently ensure that $||W||_p = 1$ when $p=2$ or $p=\infty$ in Section \ref{sec:constrained_linear_methods}. 

\paragraph{1-Lipschitz Activations: } Most common activations (such as ReLU \citep{krizhevsky2012imagenet}, tanh, maxout \citep{pmlr-v28-goodfellow13}) are 1-Lipschitz, if scaled appropriately. 

\subsection{Applications of Lipschitz Networks}

\textbf{Wasserstein Distance Estimation}
Wasserstein-1 distance (also called Earth Mover Distance) is a way to compute the distance between two probability distributions and has found many applications in machine learning \citep{peyre2018computational}. Using Kantorovich duality \citep{villani2008optimal}, one can recast the Wasserstein distance estimation problem as a maximization problem, defined over 1-Lipschitz functions:
\begin{equation}
\label{eq:dual_obj}
    W(P_{1}, P_{2}) = \sup_{f \in \lfuncs} \big(  \mathop{\mathbb{E}}_{x \sim P_{1}}[f(x)] -  \mathop{\mathbb{E}}_{x \sim P_{2}}[f(x)]  \big)
\end{equation}
\citet{arjovsky2017wasserstein} proposed the Wasserstein GAN architecture, which uses a Lipschitz network as its discriminator. 

\textbf{Adversarial Robustness}
Adversarial examples are inputs to a machine learning system which have been designed to force undesirable behaviour \citep{Szegedy2013IntriguingPO, Goodfellow2014ExplainingAH}. Given a classifier $f$ and an input $\rvx$, we can write an adversarial example as $\rvx_{adv} = \rvx + \delta$ such that $f(\rvx_{adv}) \neq f(\rvx)$ and $\delta$ is small. A small Lipschitz constant guarantees a lower bound on the size of $\delta$ \citep{tsuzuku2018lipschitz}, thus providing robustness guarantees.

\textbf{Some Other applications} Enforcing the Lipschitz constant on networks has found uses in regularization \citep{gouk2018regularisation} and stabilizing GAN training \citep{kodali2017convergence}. 

\section{Gradient Norm Preservation}
\label{sec:gnp}

When backpropagating through a norm-constrained 1-Lipschitz network, the gradient norm is non-increasing as it is processed by each layer. This leads to interesting consequences when we try to represent scalar-valued functions whose input-output gradient has norm 1 almost everywhere.
(This relates to Wasserstein distance estimation, as an optimal dual solution has this property \citep{gulrajani2017improved}.) To approximate such functions, the gradient norm must be preserved by each layer in the network during backpropagation. Unfortunately, norm-constrained networks with common activations are unable to achieve this.

\begin{restatable}{theorem}{reprproof}
\label{repr_proof}
Consider a neural net, $f: \mathbb{R}^n \rightarrow \mathbb{R}$, built with matrix 2-norm constrained weights ($||\mathbf{W}||_2 \leq 1$) and 1-Lipschitz, element-wise, monotonic activation functions. If $||\nabla f(\rvx)||_2 = 1$  almost everywhere, then $f$ is linear.
\end{restatable}

A full proof is presented in Appendix~\ref{app:repr_thm}. As a special case, Theorem~\ref{repr_proof} shows that 2-norm-constrained networks with ReLU (or sigmoid, tanh, etc.) activations cannot represent the absolute value function. For ReLU layers, gradient norm can only be preserved if every activation is positive\footnote{Except for units which don't affect the network's output. }. Hence, the network's input-output mapping must be linear.

This tension between preserving gradient norm and nonlinear processing is also observed empirically. As the Lipschitz constant is decreased, the network is forced to sacrifice nonlinear processing capacity to maintain adequate gradient norm, as discussed later in Section \ref{sec:norm_preserve_practical} and Figure~\ref{hist_and_stat}. 

Another key observation is that we may adjust all weight matrices to have singular values of 1 without losing capacity~\footnote{This condition implies gradient norm preservation. }. 
\begin{restatable}{theorem}{equivweights}
\label{equiv_weights}
Consider a network, $f: \mathbb{R}^n \rightarrow \mathbb{R}$, built with matrix 2-norm constrained weights and with $||\nabla f(\rvx)||_2 = 1$ almost everywhere. Without changing the computed function, each weight matrix $\mathbf{W} \in R^{m \times k}$ can be replaced with a matrix $\tilde{\mathbf{W}}$ whose singular values all equal 1.
\end{restatable}
The proof of Theorem~\ref{equiv_weights} is given in Appendix~\ref{app:repr_thm}. The condition of singular values equaling 1 is equivalent to the following: when $m > k$, the columns of $\tilde{\mathbf{W}}$ are orthonormal; when $m < k$, the rows of $\tilde{\mathbf{W}}$ are orthonormal; and when $m = k$, $\tilde{\mathbf{W}}$ is orthogonal. We thereon refer to such matrices as orthonormal. With these in mind, we restrict our search for expressive Lipschitz networks to those that contain orthonormal weight matrices and activations which preserve the gradient norm during backpropagation.

\section{Methods}
\label{mthd}

If we can learn any 1-Lipschitz function with a neural network then we can trivially extend this to K-Lipschitz functions by scaling the output by $K$. We thus focus on designing 1-Lipschitz network architectures with respect to the $L_{2}$ and $L_{\infty}$ metrics by requiring \emph{each} layer to be 1-Lipschitz.

\subsection{Gradient Norm Preserving Activation Functions}

\begin{figure}
\centering
\includegraphics[width=0.5\linewidth]{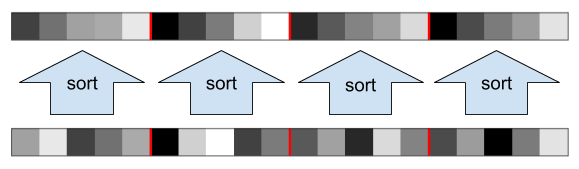}
\caption{GroupSort activation with a grouping size of 5.}
\label{f:group}
\vspace{-0.5cm}
\end{figure}

As discussed in Section~\ref{sec:gnp}, common activation functions such as ReLU are \textit{not} gradient norm preserving. To achieve norm preservation, we use a general purpose 1-Lipschitz activation we call $\bf{GroupSort}$. This activation separates the pre-activations into groups, sorts each group into ascending order, and outputs the combined "group sorted" vector (shown in Figure \ref{f:group}). Visualizations of how GroupSort transforms pre-activations are shown in Appendix \ref{app:vis_groupsort}. 

\textit{Properties of GroupSort:} GroupSort is a 1-Lipschitz operation. It is also norm preserving: its Jacobian is a permutation matrix, and permutation matrices preserve every vector $p$-norm. It is also homogeneous ($\mathbf{GroupSort}(\alpha \rvx) = \alpha \mathbf{GroupSort}(\rvx)$) as sorting order is invariant to scaling. 

\textit{Varying the Grouping Size:} When we pick a grouping size of 2, we call the operation \textbf{MaxMin}. This is equivalent to the Orthogonal Permutation Linear Unit \citep{chernodub2016norm}. When sorting the \textit{entire} input, we call the operation \textbf{FullSort}. MaxMin and FullSort are equally expressive: they can be reduced to each other without violating the norm constraint on the weights. FullSort can implement MaxMin by  "chunking" the biases in pairs. We can write:
\begin{equation*}
\mathbf{MaxMin}(\rvx) = \mathbf{FullSort}(\mathbf{I}\rvx + \vb) - \vb,
\end{equation*}
where the biases $\vb$ push each pair of activations to a different magnitude scale so that they get sorted independently (Appendix \ref{app:diff_group_size}). FullSort can also be represented using a series of MaxMin layers that implement BubbleSort; this obeys any matrix $p$-norm constraint since it can be implemented using only permutation matrices for weights. Although FullSort is able to represent certain functions more compactly, it is often more difficult to train compared to MaxMin. 

\textit{Performing Folding via Absolute Value:} Under the matrix 2-norm constraint, MaxMin is equivalent to absolute value in  expressive power, as shown in Appendix \ref{expressivity}. 

Applying absolute value to the activations has the effect of folding the space on each of the coordinate axes. Hence, a rigid linear transformation, followed by absolute value, followed by another rigid linear transformation, can implement folding along an arbitrary hyperplane. This gives an interesting interpretation of how MaxMin networks can represent certain functions by implementing absolute value, as shown in Figure~\ref{fig:folding}. \citet{montufar2014number} provide an analysis of the expressivity of networks that can perform folding. 

\begin{figure}
    \centering
    \includegraphics[width=0.8\linewidth]{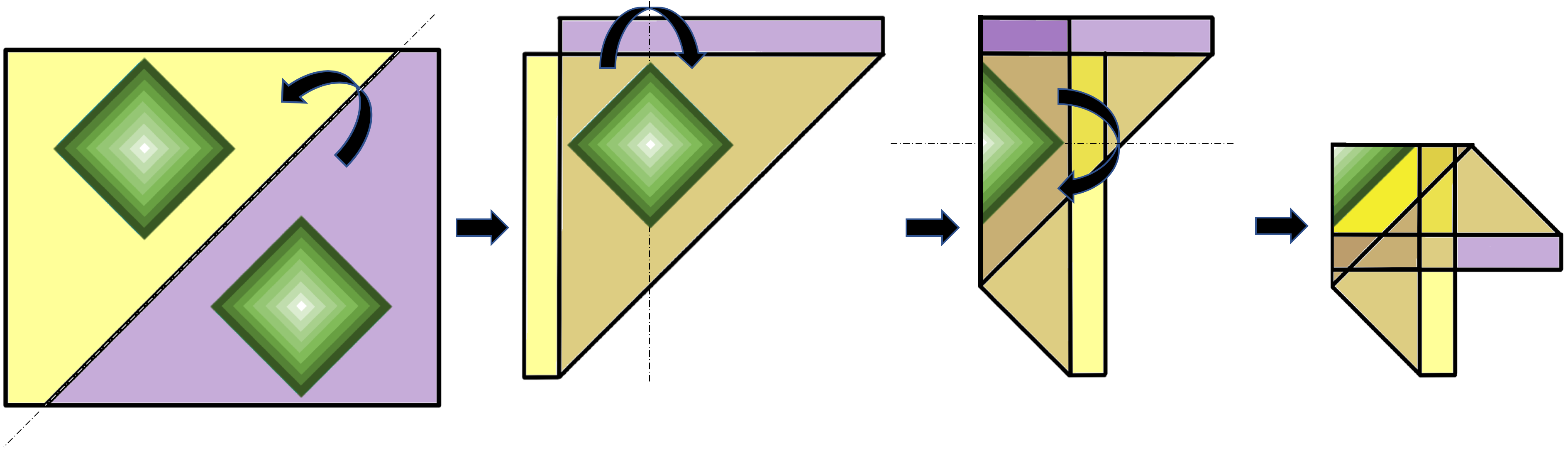}
    \caption{A rigid transformation, followed by absolute value, followed by another rigid transformation, can implement folding along an arbitrary hyperplane. Here,  the network represents a function consisting of a pair of square pyramids by three folding operations, until the function is representable as a linear function of the top layer activations.}
    \label{fig:folding}
    \vspace{-0.5cm}
\end{figure}

\textit{GroupSort and other activations:} Without norm constraints, GroupSort can recover many other common activation functions. For example, ReLU, Leaky ReLU, concatenated ReLU \citep{shang2016understanding}, and Maxout \citep{pmlr-v28-goodfellow13}. Details can be found in Appendix~\ref{app:groupsort_and_other}. 

For a discussion regarding computational considerations, refer to Appendix \ref{app:groupsort_computational}.

\vspace{-0.1cm}
\subsection{Norm-constrained linear maps}
\vspace{-0.1cm}
\label{sec:constrained_linear_methods}
We discuss how to practically enforce the 1-Lipschitz constraint on the linear layers for both 2- and $\infty$-norms.

\subsubsection{Enforcing $||W||_{2}=1$ while Preserving Gradient Norm} 
Several methods have been proposed to enforce matrix 2-norm constraints during training \citep{cisse2017parseval, yoshida2017spectral}. In the interest of preserving the gradient norm, we go a step further and enforce \emph{orthonormality} of the weight matrices. This stronger condition ensures that all singular values are exactly 1, rather than bounded by 1. 

We make use of an algorithm first introduced by \citet{bjorck1971iterative}, which we refer to as Bj\"orck Orthonormalization (or simply Bj\"orck). Given a matrix, this algorithm finds the closest orthonormal matrix through an iterative application of the Taylor expansion of the polar decomposition. Given an input matrix $A_0 = A$, the algorithm computes,
\begin{equation}\label{eqn:bjorck_update}
    A_{k+1} = A_{k}\left (I + \dfrac{1}{2}Q_k + \ldots + (-1)^p \binom{-\frac{1}{2}}{p} Q_k^p \right )
\end{equation}
where $Q_k = I - A_{k}^T A_{k}$. This algorithm is fully differentiable and thus has a pullback operator for the Stiefel manifold \citep{absil2009optimization} allowing us to optimize over orthonormal matrices directly. A larger choice of $p$ adds more computation but gives a closer approximation for each iteration. With $p=1$ around 15 iterations is typically sufficient to give a close approximation but this is computationally prohibitive for wide layers. In practice, we found that we could use 2-3 iterations per forward pass and increase this to 15 or more iterations at the end of training to ensure a tightly enforced Lipschitz constraint. There exists a simple and easy-to-implement sufficient condition to ensure convergence, as described in Appendix \ref{sec:bjorck__conv}. Bj\"orck orthonormalization was also used by \citet{van2018sylvester} to enforce orthogonal weights for variational inference with normalizing flows  \citep{rezende2015variational}. 

It is possible to project the weight matrices on the L2 ball both \textit{after} each gradient descent step, or \textit{during} forward pass (if the projection is differentiable). For the latter, the learn-able parameters of the network are unconstrained during training. In our experiments, we exploit the differentiability of Bj\"orck operations and adopt the latter approach. Note that this doesn't lead to extra computational burden at test time, as the network parameters can be orthonormalized after training, and a new network can be built using these.

Other approaches have been proposed to enforce matrix 2-norm constraints. Parseval networks \citep{cisse2017parseval} and spectral normalization \citep{miyato2018spectral} are two such approaches, each of which can be used together with the GroupSort activation. Parseval networks also aim to set all of the singular values of the weight matrices to 1, and can be interpreted as a special case of Bj\"orck orthonormalization. In Appendix~\ref{sec:bjorck_vs_parseval} we provide a comparison of the Bj\"orck and Parseval algorithms. Spectral normalization is an inexpensive and practical way to enforce the 1-Lipschitz constraint, but since it only constrains the largest singular value to be less than 1, it is not gradient norm preserving by construction. We demonstrate the practical limitations caused by this with an experiment described in Appendix~\ref{sec:bjorck_vs_spectral}. 

While we restrict our focus to fully connected layers, our analyses apply to convolutions. Convolutions can be \emph{unfolded} to be represented as linear transformations and bounding the spectral norm of the filters bound the spectral norm of the unfolded operation \citep{gouk2018regularisation, cisse2017parseval, sedghi2018singular}. For a discussion regarding computational considerations, refer to Appendix \ref{app:computational_bjorck}.

\subsubsection{Enforcing $||W||_{\infty}=1$} 
Due to its simplicity and suitability for a GPU implementation, we use Algorithm 1 from \citet{condat2016fast} (see Appendix \ref{app:proj_linf}) to project the weight matrices onto the $L_{\infty}$ ball. 

\subsection{Provable Adversarial Robustness}\label{sec:prov_adv}

A small Lipschitz constant limits the change in network output under small adversarial perturbations. As explored by \citet{tsuzuku2018lipschitz}, we can guarantee adversarial robustness at a point by considering the \textit{margin} about that point divided by the Lipschitz constant. Formally, given a network with Lipschitz constant $K$ (with respect to the $L_\infty$ metric) and an input $\rvx$ with corresponding class $t$ that produces logits $\rvy$, we define its margin by 
\begin{equation}\label{eqn:margin}
    \mathcal{M}(\rvx) = \max(0, y_{t} - \max_{i \neq t} y_i)
\end{equation}
If $\mathcal{M}(\rvx) > K \epsilon / 2$, the network is robust to all perturbations $\delta$ with $||\delta||_\infty < \epsilon$, at $\rvx$. We train our networks with $\infty$-norm constrained weights using a multi-class hinge loss:
\begin{equation}\label{eqn:multi_margin}
L(\rvy, t) = \sum_{i \neq t} \max(0, \kappa - (y_t - y_i))
\end{equation}
where $\kappa$ controls the margin enforcement and depends on the Lipschitz constant and desired perturbation tolerance. 

\subsection{Dynamical Isometry and Preventing Vanishing Gradients}
\label{sec:dyn_iso}
Gradient norm preserving networks can also represent functions whose input-output Jacobian has singular values that all concentrate near unity \citep{pennington2017resurrecting}, a property known as \textit{dynamical isometry}. This property has been shown to speed up training by orders of magnitude when enforced during weight initialization \citep{pennington2017resurrecting}, and explored in the contexts of training RNNs \citep{chen2018dynamical} and deep CNNs \citep{xiao2018dynamical}. Enforcing norm preservation on each layer also solves the vanishing gradients problem, as the norm of the back-propagated gradients are maintained at unity. Using our methods, one can achieve dynamical isometry throughout training (see  Figure \ref{fig:all_sn_hist}), reaping the aforementioned benefits. Note that ReLU networks cannot achieve dynamical isometry \citep{pennington2017resurrecting}. We leave exploring these benefits to a future study.

\section{Related Work}
\label{rltd}

Several methods have been proposed to train Lipschitz neural networks \citep{cisse2017parseval, yoshida2017spectral, miyato2018spectral, gouk2018regularisation}. \citet{cisse2017parseval} regularize the weights of the neural network to obey an orthonormality constraint. The corresponding update to the weights can be seen as one step of the Bj\"orck orthonormalization scheme (Eq.~\ref{eqn:bjorck_update}). This regularization can be thought of as projecting the weights closer to the manifold of orthonormal matrices after each update. This is a critical difference to our own work, in which a differentiable projection is used during each update. Another approach, spectral normalization \citep{miyato2018spectral}, employs power iteration to rescale each weight by its spectral norm. Although efficient, spectral normalization doesn't guarantee gradient norm preservation and can therefore under-use Lipschitz capacity, as discussed in Appendix \ref{sec:bjorck_vs_spectral}. \citet{arjovsky2016unitary, wisdom2016full, sun2017learning} use explicitly parametrized square orthogonal weight matrices. 

Other techniques penalize the Jacobian of the network, constraining the Lipschitz constant locally \citep{gulrajani2017improved, drucker1992improving, sokolic2017robust}. While it is often easy to train networks under such penalties, these methods don't provably enforce a Lipschitz constraint. 

The Lipschitz constant of neural networks has been connected theoretically and empirically to generalization performance \citep{bartlett1998sample, bartlett2017spectrally, neyshabur2017exploring, neyshabur2018a, sokolic2017robust}. \citet{neyshabur2018a} show that if the network Lipschitz constant is small then a non-vacuous bound on the generalization error can be derived. Small Lipschitz constants have also been linked to adversarial robustness \citep{tsuzuku2018lipschitz, cisse2017parseval}. In fact, adversarial training can be viewed as approximate gradient regularization \citep{miyato2017virtual, simon2018adversarial} which makes the function Lipschitz locally around the training data. Lipschitz constants can also be used to provide provable adversarial robustness guarantees. \citet{tsuzuku2018lipschitz} manually enforce a margin depending on an approximation of the upper bound on the Lipschitz constant which in turn guarantees adversarial robustness. In this work we also explore provable adversarial robustness through margin training but do so with a network whose Lipschitz constant is known and globally enforced.

Classic neural network universality results use constructions which violate the norm-constraints needed for Lipschitz guarantees \citep{cybenko1989approximation, hornik1991approximation}. \citet{huster2018limitations} explored universal approximation properties of Lipschitz networks and proved that ReLU networks cannot approximate absolute value with $\infty$-norm constraints. In this work we also show that many activations, including ReLU, are not sufficient with $2$-norm constraints. We prove that Lipschitz functions \emph{can} be universally approximated if the correct activation function is used.

\section{Universal Approximation of Lipschitz Functions}

Universal approximation results for continuous functions don't apply to Lipschitz networks as the constructions typically involve large Lipschitz constants. Moreover, \citet{huster2018limitations} showed that it is impossible to approximate even the absolute value function with $\infty$-norm-constrained ReLU networks. We now present theoretical guarantees on the approximation of Lipschitz functions. To our knowledge, this is the first universal Lipschitz function approximation result for norm-constrained networks.

We will first prove a variant of the Stone-Weierstrass Theorem which gives a simple criterion for universality (similar to Lemma 4.1 in \citet{yaacov2010lipschitz}). We then construct a class of networks with GroupSort which satisfy this criterion. 
\begin{definition}
We say that a set of functions, $L$, is a \emph{lattice} if for any $f,g \in L$ we have $max(f,g) \in L$ and $min(f,g) \in L$ (where $max$ and $min$ are defined pointwise).
\end{definition}
\begin{restatable}{lemma}{swtheorem}\label{lemma:stoneweierstrass}
(Restricted Stone-Weierstrass Theorem) Suppose that $(X, d_{X})$ is a compact metric space with at least two points and $L$ is a lattice in $\lfuncs$ with the property that for any two distinct elements $x,y \in X$ and any two real numbers $a$ and $b$ such that $|a-b| \leq d_{X}(x,y)$ there exists a function $f \in L$ such that $f(x) = a$ and $f(y) = b$. Then $L$ is dense in $\lfuncs$.
\end{restatable}
\begin{remark}
We could replace $|\cdot|$ with any metric on $\reals$.
\end{remark}

The full proof of \Lemref{lemma:stoneweierstrass} is presented in Appendix \ref{app:univ}. Note that Lemma \ref{lemma:stoneweierstrass} says that $\mathcal{A}$ is a universal approximator for 1-Lipschitz functions iff $\mathcal{A}$ is a lattice that separates points. Using \Lemref{lemma:stoneweierstrass}, we can derive the second of our key results. Norm-constrained networks with GroupSort activations are able to approximate any Lipschitz function in $L_p$ distance.

\begin{figure}
    \centering
    \includegraphics[width=0.9\linewidth]{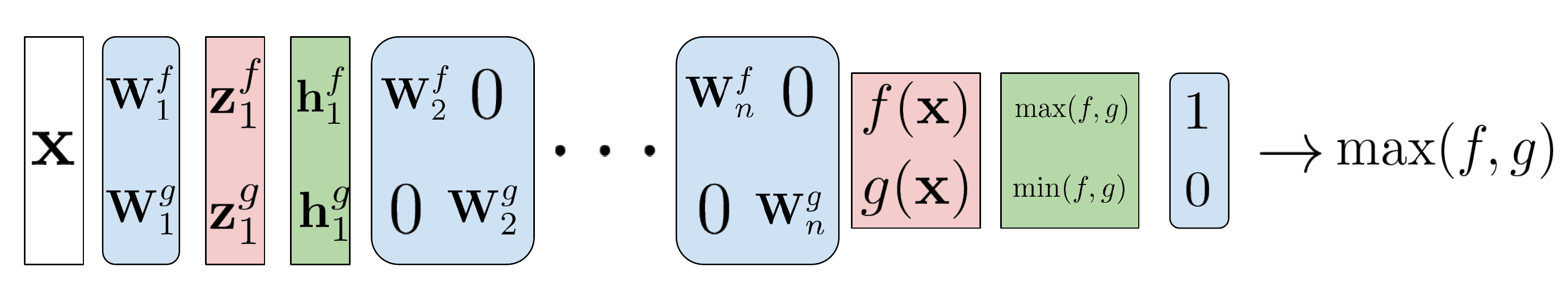}
    \caption{Lattice construction for universal approximation.}
    \label{fig:univ_construction}
    \vspace{-0.5cm}
\end{figure}

\begin{restatable}{theorem}{lpuniv}\label{thm:universalapprox} (Universal Approximation with Lipschitz Networks)
Let $\mathcal{LN}_{p}$ denote the class of fully-connected networks whose first weight matrix satisfies $||\rmW_{1}||_{p,\infty} = 1$, all other weight matrices satisfy $||\rmW||_{\infty}=1$, and GroupSort activations have a group size of 2. Let $X$ be a closed and bounded subset of $\reals^n$ endowed with the $L_p$ metric. Then the closure of $\mathcal{LN}_{p}$ is dense in $\lfuncs$.
\end{restatable}
\begin{proof} (Sketch)
Observe first that $\mathcal{LN}_{p} \subset \lfuncs$. By \Lemref{lemma:stoneweierstrass}, it is sufficient to show that $\mathcal{LN}_{p}$ is closed under max and min and has the point separation property. For the latter, given $x,y \in X$ and $a,b \in \reals$ with $|a-b| \leq ||x-y||_p$, we can fit a line with a single layer network, $f$, satisfying the 1-Lipschitz constraint with $f(x)=a$ and $f(x)=b$.

Now consider $f$ and $g$ in $\mathcal{LN}_{p}$. For simplicity, assume that they have the same number of layers. We can construct the layers of another network $h \in \mathcal{LN}_{p}$ by vertically concatenating the weight matrices of the first layer in $f$ and $g$, followed with block diagonal matrices constructed from the remaining layers of $f$ and $g$ (see Figure~\ref{fig:univ_construction}). The final layer of the network computes $[f(x), g(x)]$. We then apply GroupSort to get $[max(f,g)(x), min(f,g)(x)]$ and take the dot product with $[1,0]$ or $[0,1]$ to get the max or min.
\end{proof}
The formal proof of \Thmref{thm:universalapprox} is presented in Appendix \ref{app:univ}. One special case of \Thmref{thm:universalapprox} is for 1-Lipschitz functions in $L_\infty$ norm, in this case we may extend the restricted Stone-Weierstrass theorem in $L_\infty$ norm to vector-valued functions to  prove universality in this setting. Formally:
\begin{restatable}{observation}{linfuniv}
Consider the set of networks, $\mathcal{LN}^{m}_{\infty} = \{ f: \reals^n \rightarrow \reals^m, ||W||_\infty = 1 \}$. Then $\mathcal{LN}^{m}_{\infty}$ is dense in 1-Lipschitz functions with respect to the $L_\infty$ metric.
\end{restatable}
While these constructions rely on constraining the $\infty$-norm of the weights \footnote{Our construction fails for 2-norm constrained weights, as column-wise stacking two matrices that have max singular values of 1 might result in a matrix that has singular values larger than 1. }, constraining the 2-norm often makes the networks easier to train, and we have not yet found a Lipschitz function which 2-norm constrained GroupSort networks couldn't approximate empirically. It remains an open question whether 2-norm constrained GroupSort networks are also universal Lipschitz function approximators.

\section{Experiments}

\label{exprmt}
\begin{figure}[t!]
\centering
\includegraphics[width=0.68\linewidth]{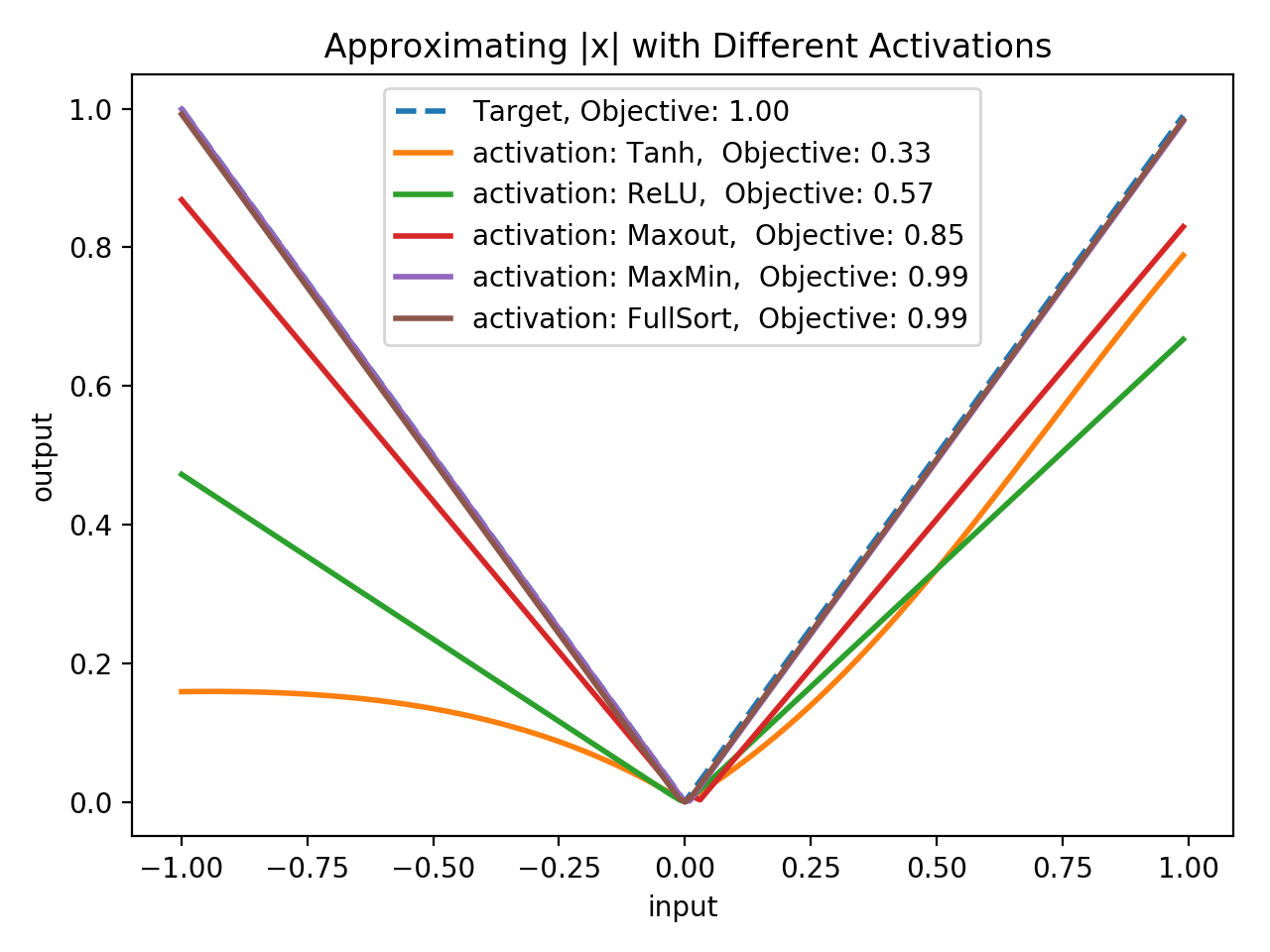}
\caption{Approximating the absolute value function with Lipschitz networks using different activations. The objective values indicate the estimated Wasserstein Distance. }
\label{f:abs}
\vspace{-0.5cm}
\end{figure}
Our experiments had two main goals. First, we wanted to test whether norm-constrained GroupSort architectures can represent Lipschitz functions other approaches cannot.
Second, we wanted to test if our networks can perform competitively with existing approaches on practical tasks that require strict bounds on the global Lipschitz constant. 
We present additional results in Appendix~\ref{sec:additional_exp}, including CIFAR-10 \citep{krizhevsky2009learning} classification and MNIST small data classification. Experiment details are shown in Appendix~\ref{app:exp_details}.

\subsection{Representational Capacity}
We investigate the ability of 2-norm-constrained networks with different activations to represent Lipschitz functions. 

\subsubsection{Quantifying Expressive Power}
\label{sec:optimal_dual}
\begin{figure}[tp]
    \centering
    \includegraphics[width=0.60\linewidth]{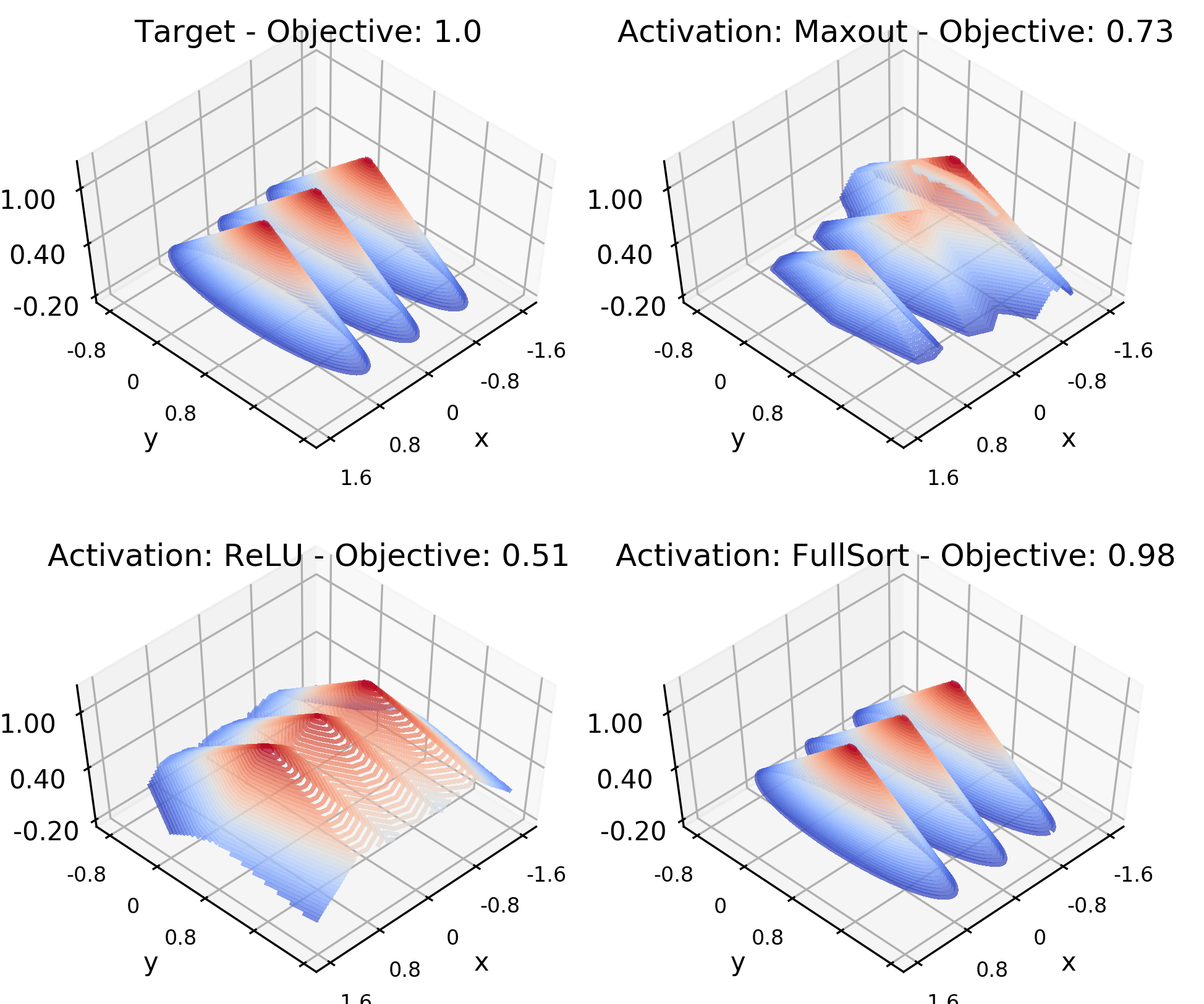}
    \caption{Approximating three circular cones with slope 1 using Lipschitz networks using various activations. The objective values represent the estimated Wasserstein distance.}
    \label{f:3_cone}
    \vspace{-0.5cm}
\end{figure}
We propose an effective method to quantify how expressive different Lipschitz architectures are. We pick pairs of probability distributions whose Wasserstein Distance and optimal dual surfaces can be computed analytically. We then train networks under the Wasserstein distance objective (Equation \ref{eq:dual_obj}) using samples from these distributions to assess how closely they can estimate the Wasserstein distance. For 1D and 2D problems, we visualize the learned dual surfaces to inspect the failure modes of non-expressive architectures. 

We focus on approximating the absolute value function, three circular cones and circular cones in higher dimensions. Appendix \ref{app:surf} describes how pairs of probability distributions can be picked which have these optimal dual surfaces, and a Wasserstein distance of precisely 1. 

Figure \ref{f:abs} shows the functions learned by Lipschitz networks built with various activations, trained to approximate absolute value. Non-GNP (non-gradient norm preserving) activations are incapable of approximating this trivial Lipschitz function. While increasing the network depth helps (Table~\ref{t:prob_dim}), this representational barrier leads to limitations as the problem dimensionality increases.

Figure \ref{f:3_cone} shows the dual surfaces approximated by networks trained to approximate three circular cones. This figure points to an even more serious pathology with non-GNP activations: by attempting to increase the slope, the non-GNP networks may distort the shape of the function, leading to different behavior from the optimal solution. In the case of training WGAN critics, this cannot be fixed by increasing the Lipschitz constant. (Optimal critics for different Lipschitz constants are equivalent up to scaling.)


We evaluated the expressivity of architectures built with different activations for higher dimensional inputs, on the task of approximating high dimensional circular cones. As shown in Table \ref{t:prob_dim}, increasing problem dimensionality leads to significant drops in the Wasserstein objective for networks built with non-GNP activations, and increasing the depth of the networks only slightly improves the situation. We also observed that while the MaxMin activation performs significantly better, it also needs large depth to learn the optimal solution. Surprisingly, shallow FullSort networks can easily approximate high dimensional circular cones.

\begin{table}[bp]
\centering
\begin{tabular}{lllllll}
\multicolumn{1}{c}{Input Dim. }  &\multicolumn{1}{c}{ 128}&\multicolumn{1}{c}{ 128}&\multicolumn{1}{c}{ 256}&\multicolumn{1}{c}{256}&\multicolumn{1}{c}{ 512}&\multicolumn{1}{c}{512}   \\
\multicolumn{1}{c}{Depth}  &\multicolumn{1}{c}{\it 3}&\multicolumn{1}{c}{\it 7 }&\multicolumn{1}{c}{\it 3}&\multicolumn{1}{c}{\it 7 }&\multicolumn{1}{c}{\it 3}&\multicolumn{1}{c}{\it 7 }
\\ \hline \\
ReLU                & 0.51 & 0.60 & 0.50 & 0.53 & 0.46 & 0.49  \\
Maxout              & 0.66 & 0.71 & 0.60 & 0.66 & 0.52 & 0.56  \\
MaxMin              & 0.87 & 0.95 & 0.83 & 0.93 & 0.72 & 0.88  \\
FullSort            & \textbf{1.00} & \textbf{1.00} & \textbf{1.00} & \textbf{1.00} & \textbf{1.00} & \textbf{1.00} 
\end{tabular}
\caption{\textbf{Effect of problem dimensionality:} Testing how well different activations and depths can optimize the dual Wasserstein objective with different input dimensionality. The optimal dual surface obtains a dual objective of 1.}
\label{t:prob_dim}
\end{table}

\subsubsection{Relevance of Gradient Norm Preservation in Practical Settings}
\label{sec:norm_preserve_practical}
Thus far, we have focused on examples where the gradient of the network should be 1 almost everywhere. For many practical tasks we don't need to meet this strong condition. Is gradient norm preservation relevant in other settings?

\paragraph{How much of the Lipschitz capacity can we use?} We have proven that ReLU networks approach linear functions as they utilize the full gradient capacity allowed with Lipschitz constraints. To understand these implications practically, we trained ReLU and GroupSort networks on MNIST with orthonormal weight constraints enforced to ensure that they are 10-Lipschitz functions. We looked at the distribution of the spectral radius (largest singular value) of the network Jacobian over the training data. Figure~\ref{hist_and_stat} displays this distribution for each network. We observed that while both networks satisfy the Lipschitz constraint, the GroupSort network does so much more tightly than the ReLU network. The ReLU network was not able to make use of the capacity afforded to it and the observed Lipschitz constant was actually closer to 8 than 10. In Appendix~\ref{app:dyn_iso} we show the full singular value distribution which suggests that 2-norm-constrained GroupSort networks can achieve near-dynamical isometry throughout training.

We studied the activation statistics of ReLU networks trained on MNIST with and without Lipschitz constraints in Figure~\ref{hist_and_stat}. Given a threshold, $\tau \in [0,1]$, we computed the proportion of activations throughout the network which are positive at least as often as $\tau$ over the training data. Without a Lipschitz constraint, the activation statistics were much sparser, with almost no units active when $\tau > 0.4$. Smaller Lipschitz constants forced the network to use more positive activations to make use of its gradient capacity (see Section~\ref{sec:gnp}). In the worst case, about 10\% of units were ``undead'', or active all of the time, and hence didn't contribute any nonlinear processing. It's not clear what effect this has on representational capacity, but such a dramatic change in the network's activation statistics suggests that it made significant compromises to maintain adequate gradient norm.

\begin{figure}
\begin{minipage}{0.45\linewidth}
    \includegraphics[width=\linewidth]{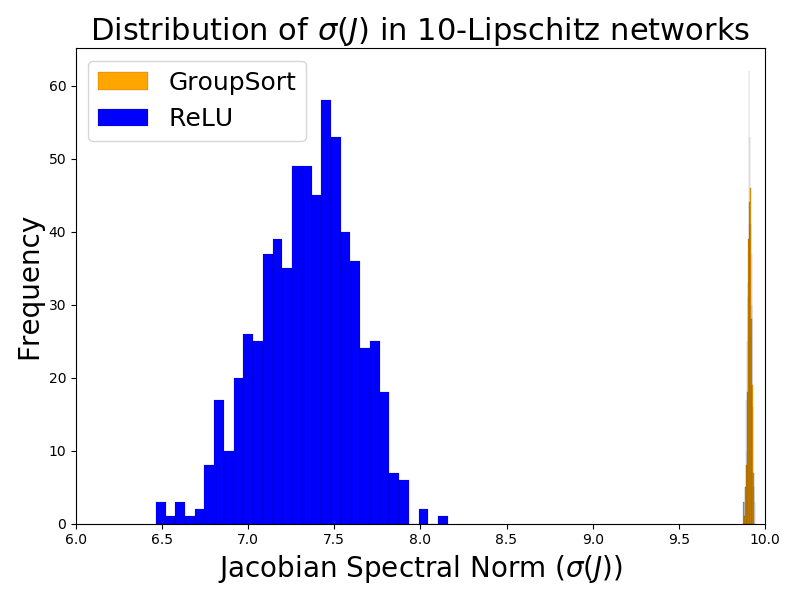}
\end{minipage}\hfill
\begin{minipage}{0.45\linewidth}
\includegraphics[width=\linewidth]{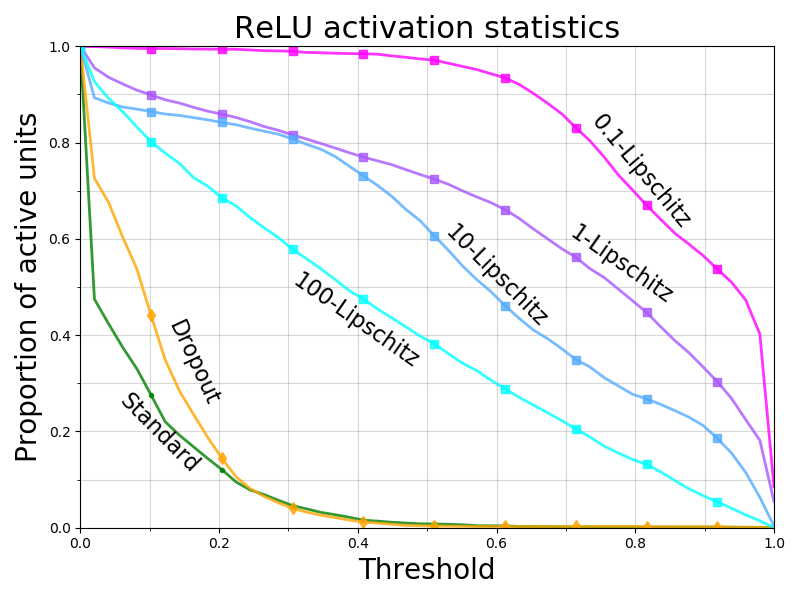}
\end{minipage}\hfill%
\vspace{-0.2cm}
\caption{\textbf{(left) Jacobian spectral norm distribution:} Jacobian spectral norm distribution of 10 Lipschitz ReLU and GroupSort nets. \textbf{(right) Activation statistics Lipschitz ReLU nets:} Ratio of activations that are positive more often than the threshold. }
\vspace{-0.5cm}
\label{hist_and_stat}
\end{figure}

\subsection{Wasserstein Distance Estimation}


\begin{table}[b!]
\begin{center}
\begin{tabular}{ll|ll}
  &Linear & \textbf{MNIST} & \textbf{CIFAR10}\\ \hline
ReLU & Spectral & $0.95 \pm 0.01$ & $1.12 \pm 0.02$\\
Maxout & Spectral & $1.20 \pm 0.03$ & $1.40 \pm 0.01$ \\ 
MaxMin & Spectral & $1.36 \pm 0.07$ & $1.62 \pm 0.04$\\
GroupSort(4) & Spectral & $1.64 \pm 0.02$ & $1.63 \pm 0.03$\\
GroupSort(9) & Spectral & $1.70 \pm 0.02$ & $1.41 \pm 0.04$\\
ReLU & Bj\"orck & $1.40 \pm 0.01$ & $1.39 \pm 0.01$\\
Maxout & Bj\"orck & $1.95 \pm 0.01$ & $1.76 \pm 0.02$ \\ 
MaxMin & Bj\"orck & $2.16 \pm 0.01$ & $2.08 \pm 0.02$\\
GroupSort(4) & Bj\"orck & $\mathbf{2.31 \pm 0.01}$ & $2.17 \pm 0.02$\\
GroupSort(9) & Bj\"orck & $2.31 \pm 0.01$ & $\mathbf{2.23 \pm 0.02}$\\
\end{tabular}
\caption{Estimating the Wasserstein Distance between the data and generator distributions using 1-Lipschitz feedforward networks, for MNIST and CIFAR10 GANs.  }
\label{t:wde}
\end{center}
\end{table}
We have shown that our methods can obtain tighter lower bounds on Wasserstein distance on synthetic tasks in Section \ref{sec:optimal_dual}. We now consider the more challenging task of computing the Wasserstein distance between the generator distribution of a GAN and the empirical distribution of the data it was trained on.\footnote{The Wasserstein distance to the empirical data distribution is likely to be a loose upper bound on the Wasserstein distance to the data generating distribution, but this task still tests the ability to estimate Wasserstein distance in high-dimensional spaces.} As the optimal surfaces under the dual Wasserstein objective have a gradient norm of 1 almost everywhere (Corollary 1 in \citet{gemici2018primal}), the gradient norm preservation properties discussed in Section \ref{sec:gnp} are critical. Experiment details are outlined in Appendix \ref{app:wde}.

We trained a GAN variant on MNIST and CIFAR10 datasets, then froze the weights of the generators. Using samples from the generator and original data distribution, we trained independent 1-Lipschitz networks to compute the Wasserstein distance between the empirical data distribution and the generator distribution. We used a shallow fully connected architecture (3 layers, 720 neurons wide). As seen in Table \ref{t:wde}, using norm-preserving activation functions helps achieve tighter lower bounds on Wasserstein distance.

\textit{Training WGANs:} We were also able to train Wasserstein GANs \citep{arjovsky2017wasserstein} whose discriminators comprised of networks built with our proposed proposed 1-Lipschitz building blocks. Some generated samples can be found in Appendix \ref{app:wgan}. We leave further investigation of the GANs built with our techniques to a future study.
\subsection{Robustness and Interpretability}

\begin{figure}
\centering
  \centering
  \includegraphics[width=0.62\linewidth]{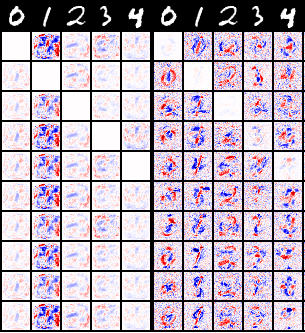}
\caption{Gradients of input images with respect to targeted cross-entropy loss, for standard (left) and Lipschitz (right) nets. Images from classes 0-4 were chosen randomly from the test set (first row). Following rows show the gradient with different targets (0-9). Positive pixel values are red. \vspace{-0.1cm}} 
\label{fig:robust_grads}
\end{figure}

We explored the robustness of Lipschitz networks trained on MNIST to adversarial perturbations measured with $L_\infty$ distance. We enforced an $L_\infty$ constraint on the weights and used the multi-class hinge loss (Equation~\ref{eqn:multi_margin}), as we found this to be more effective than manual margin training \citep{tsuzuku2018lipschitz}. We enforced a Lipschitz constant of $K=1000$ and chose the margin $\kappa = Ka$ where $a$ was $0.1$ or $0.3$. This technique provides margin-based provable robustness guarantees as described in Section~\ref{sec:prov_adv}. We also compared to PGD training \citep{madry2017towards}. We attacked each model using the FGS and PGD methods (using random restarts and 200 iterations for the latter) \citep{Szegedy2013IntriguingPO, madry2017towards} under the CW loss \citep{carlini2016towards}. The results are presented in Figure~\ref{fig:adv_robustness}. The Lipschitz networks with MaxMin activations achieved better clean accuracy and larger margins than their ReLU counterparts, leading to better robustness. PGD training requires large capacity networks \citet{madry2017towards} and we were unable to match the large perturbation performance of margin training with this architecture (using a larger CNN would produce better results). Note that the Lipschitz networks don't see any adversarial examples during training.

\begin{figure}[t!]
    \centering
    \includegraphics[width=8.3cm, height=2.80cm]{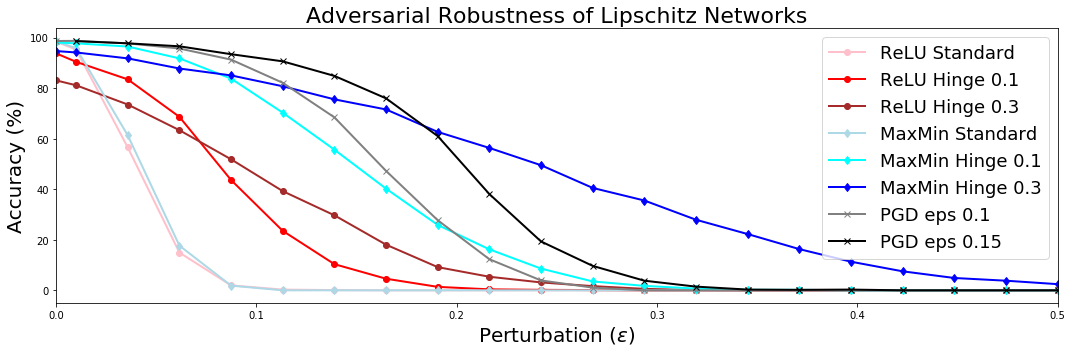}
    \caption{\textbf{Adversarial Robustness} Accuracy on PGD adversarial examples for varying perturbation sizes $\epsilon$.}
    \label{fig:adv_robustness}
    \centering
    \includegraphics[width=8.3cm, height=2.80cm]{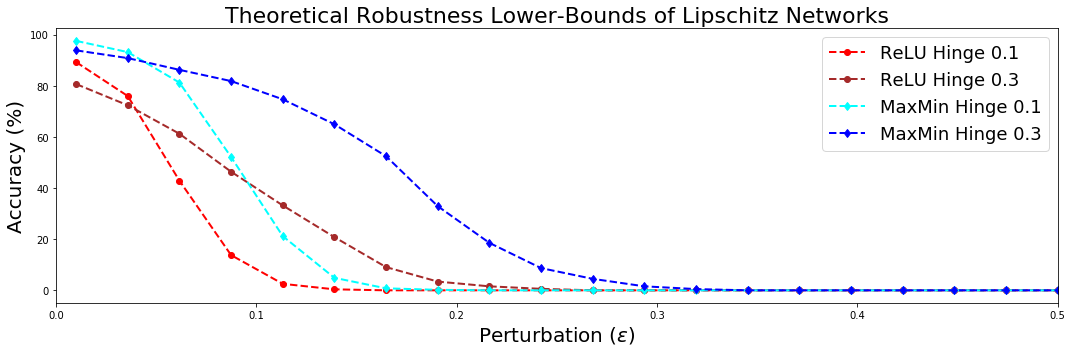}
    \caption{\textbf{Theoretical Adversarial Robustness} Theoretical accuracy lower bound for perturbation sizes $\epsilon$.}
    \label{fig:adv_robustness_theory}
    \vspace{-0.3cm}
\end{figure}

With the strictly enforced Lipschitz constant, we can compute theoretical lower bounds on the accuracy against adversaries with a maximum perturbation strength $\epsilon$. In Figure~\ref{fig:adv_robustness_theory}, we show this lower bound for each of the models previously studied. This is computed by finding the proportion of data points which violate the margin by at least $K\epsilon$. Note that at the computed threshold, the model has low confidence in the adversarial example. An even larger perturbation would be required to induce confident misclassification.

Adversarially trained networks learn robust features and have interpretable gradients \citep{tsipras2018there}. We found that this holds for Lipschitz networks, without using adversarial training. The gradients with respect to the inputs are displayed for a standard network and a Lipschitz network (with 2-norm constraints) in Figure~\ref{fig:robust_grads}.

\section{Conclusion}

We identified gradient norm preservation as a critical component of Lipschitz network design and showed that failure to achieve this leads to less expressive networks. By combining the GroupSort activation and orthonormal weight matrices, we presented a class of networks which are provably 1-Lipschitz and can approximate any 1-Lipschitz function arbitrarily well. Empirically, we showed that our GroupSort networks are more expressive than existing architectures and can be used to achieve better estimates of Wasserstein distance and provable adversarial robustness guarantees. 
\clearpage

\subsubsection*{Acknowledgments}
We extend our warm thanks to our colleagues for many helpful discussions. In particular, we would like to thank Mufan Li for pointing us towards the lattice formulation of the Stone-Weierstrass theorem, and Qiyang Li for his help in correcting a minor issue with the robustness experiments. We also thank Elliot Creager, Ethan Fetaya, J\"orn Jacobsen, Mark Brophy, Maryham Mehri Dehnavi, Philippe Casgrain, Saeed Soori, Xuchan Bao and many others not listed here for draft feedback and many helpful conversations. 

\bibliography{references}
\bibliographystyle{icml2019}

\clearpage
\pagebreak

\begin{appendices}
\section{GroupSort Activation}\label{app:activations}
In this section, we provide visualizations to shed light on how GroupSort networks compute simple 1D functions, explain how GroupSort compares with other activations, analyze the effect of the grouping size on its expressivity and discuss its computational complexity of GroupSort. 

\subsection{Visualizing GroupSort Networks}
\label{app:vis_groupsort}
In Figures \ref{f:abs_vis} and \ref{f:two_kinks}, we visualize the hidden layer activations of GroupSort networks as the input to the network is varied. The networks are approximating the absolute value function and a curve resembling the letter "W", with a slope of 1 almost everywhere. 

\begin{figure}[bh!]
\centering
\includegraphics[width=1\linewidth]{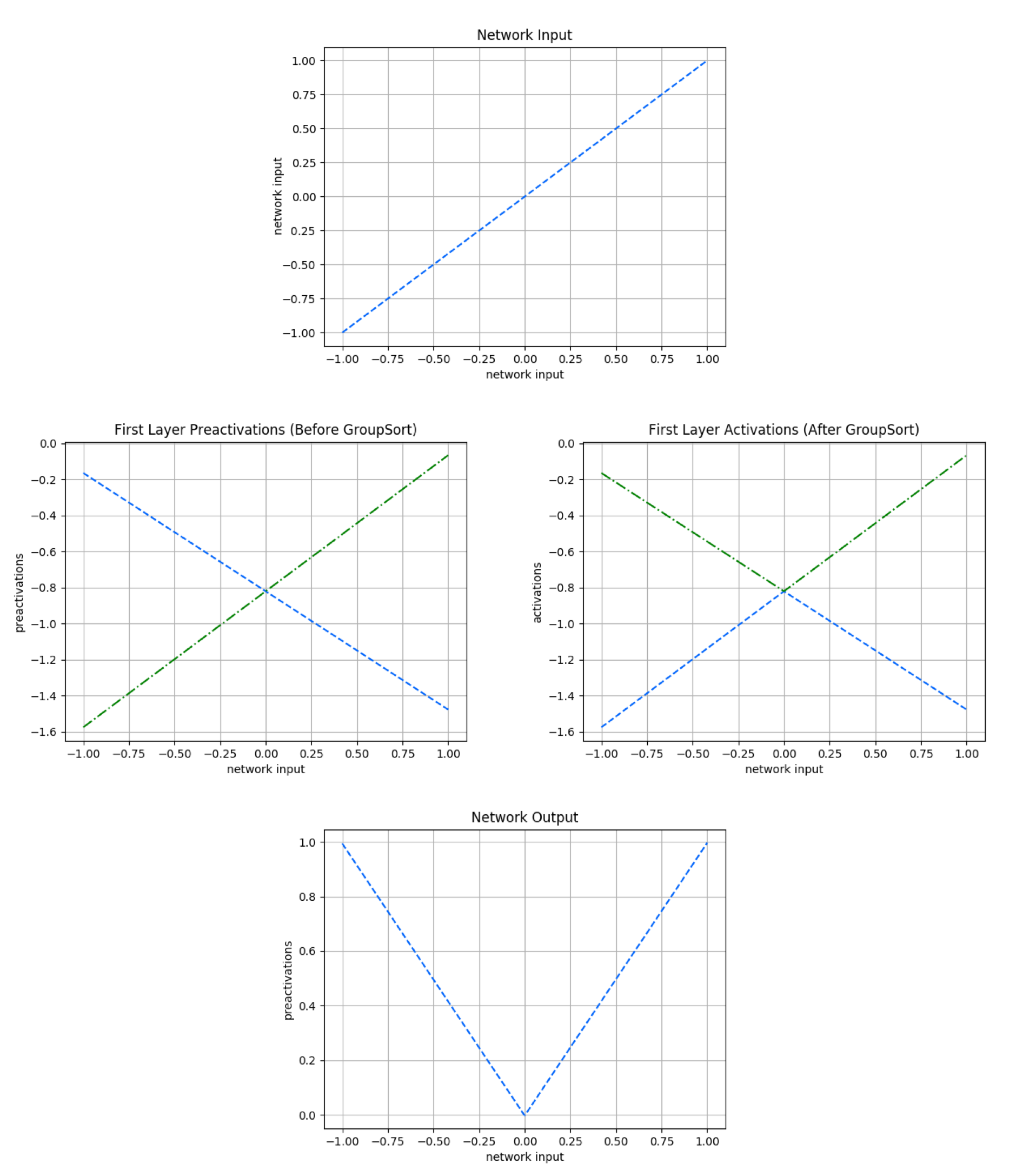}
\caption{Visualization of the pre-activations and activations of a one hidden layer GroupSort network that is approximating the absolute value function. The network has two units in its hidden layer. }
\label{f:abs_vis}
\end{figure}

\begin{figure}[th!]
\centering
\includegraphics[width=1\linewidth]{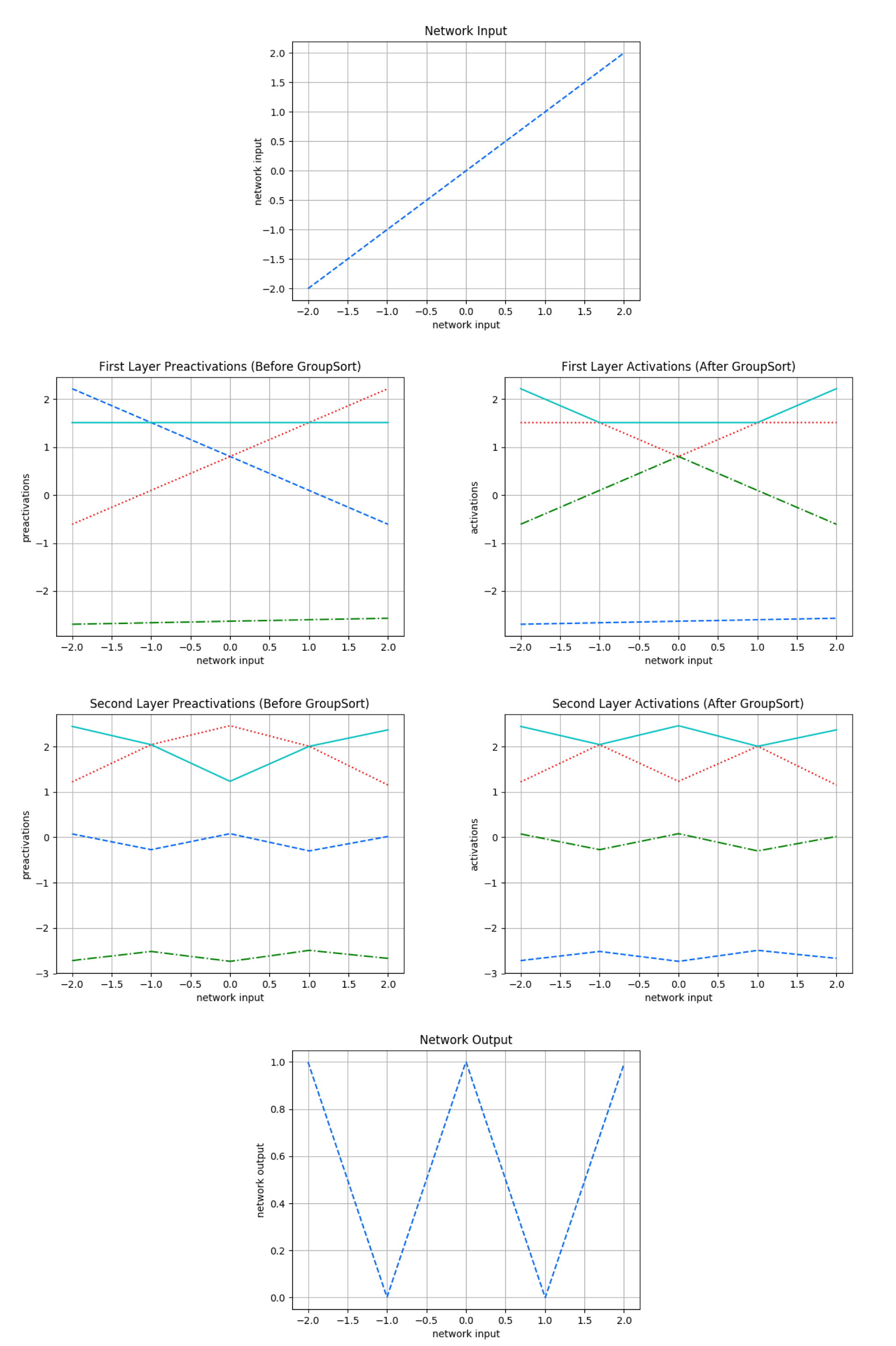}
\caption{Visualization of the pre-activations and activations of a two hidden layer GroupSort network that is approximating a curve resembling the letter "W" with a slope of 1 almost everywhere. The network has four units in its hidden layers. }
\label{f:two_kinks}
\end{figure}

\subsection{GroupSort and other activations} 
\label{app:groupsort_and_other}
Here we show that GroupSort can recover ReLU, Leaky ReLU, concatenated ReLU, and maxout activation functions. We first show that MaxMin can recover ReLU and its variants. Note that,
\begin{equation}
    \mathbf{MaxMin}(
\left[\begin{array}{c}
x \\
0
\end{array}\right]) = \left[\begin{array}{c}
ReLU(x) \\
-ReLU(-x)
\end{array}\right]
\end{equation}
By inserting $0$ elements into the pre-activations and then applying another linear transformation after MaxMin we can output either ReLU or concatenated ReLU. Explicitly,
\begin{equation}\label{eqn:maxmin_to_relu}
    \left[1 \quad 0\right] \mathbf{MaxMin}(
\left[\begin{array}{c}
x \\
0
\end{array}\right]) = ReLU(x)
\end{equation}
\begin{equation}
    \left[\begin{array}{cc}
1 & 0 \\
0 & -1
\end{array}\right] \mathbf{MaxMin}(
\left[\begin{array}{c}
x \\
0
\end{array}\right]) = \left[\begin{array}{c}
ReLU(x) \\
ReLU(-x)
\end{array}\right]
\end{equation}
If instead of adding $0$ to the preactivations we added $ax$ we could recover Leaky ReLU by using a linear transformation to select $\max(x,ax)$ (similarly to Equation~\ref{eqn:maxmin_to_relu}).

To recover maxout with groups of size $k$, we perform GroupSort with groups of size $k$ and use the next linear transformation to select the first element of each group after sorting. 

\subsection{Expressivity of GroupSort}
We show that GroupSort activation with different grouping sizes have the same expressive power. We also show that neural networks built with GroupSort activation and absolute value activation have the same expressive power. 

\textbf{Expressivity of Different Grouping Sizes}\label{app:diff_group_size}
FullSort can implement MaxMin by  "chunking" the biases in pairs. To be more precise, let $x_{max} = \sup_{\rvx \in \mathcal{X}} ||\rvx||_{\infty}$ where $\mathcal{X}$ represents the domain, and $\vb = [b_{1}, b_{2}, ... ,b_{n}]^{T}$ where $x_{max} < b_{1} = b_{2} \ll b_{3} = b_{4} \ll \dots \ll b_{n-1} = b_{n}$ ($\ll$ denotes differing by at least $x_{max}$). We can write: 
\[
\mathbf{MaxMin}(\rvx) = \mathbf{FullSort}(\mathbf{I}\rvx + \vb) - \vb,
\]
where $\mathbf{I}$ denotes the identity matrix.

\textbf{Expressivity of GroupSort and Absolute Value Networks}
\label{expressivity}
Under the matrix 2-norm constraint, neural neural networks built with GroupSort activation and absolute value activation have the same expressive power. The two operations can be written in terms of each other, as can be seen below: 

\begin{gather*}
\label{abs_to_maxmin}
\left[\begin{array}{cc}
\mathbf{max}(x) \\
\mathbf{min}(y)
\end{array}\right] 
= 
\mathbf{M} \text{ }
\mathbf{abs}(
\mathbf{M} 
\left[\begin{array}{cc}
x \\
y
\end{array}\right]  + 
\left[\begin{array}{cc}
B \\
0
\end{array}\right] ) - 
\left[\begin{array}{cc}
\sqrt{2}B \\
0
\end{array}\right] \\
\text{where} \\
\mathbf{M} = \left[\begin{array}{cc}
\frac{1}{\sqrt{2}} \quad \frac{1}{\sqrt{2}} \\
\frac{1}{\sqrt{2}} \quad \frac{-1}{\sqrt{2}}
\end{array}\right]
\end{gather*}
\begin{equation*}
\label{maxmin_to_abs}
    \mathbf{abs}(x) = 
        \left[\begin{array}{cc}
        \frac{1}{\sqrt{2}} \quad -\frac{1}{\sqrt{2}}
        \end{array}\right]
        \mathbf{MaxMin}(
        \left[\begin{array}{cc}
        \frac{1}{\sqrt{2}} \\ \frac{-1}{\sqrt{2}}
        \end{array}\right]x
        )
\end{equation*}
In Equation \ref{abs_to_maxmin}, the value of $B$ is chosen such that $2\rvx + \sqrt{2}B > 0$ for all $\rvx$ in the domain. Note that all the matrices in these constructions satisfy the matrix 2-norm constraint. 

\subsection{Computational Considerations}
\label{app:groupsort_computational}
Let $n$ be the total number of pre-activations and $k$ be the size of the groups used in GroupSort. Then, a naive CPU implementation of GroupSort has a complexity of $\frac{n}{k}\mathcal{O}(k \log{k})$. However, this operation can be parallelized on GPU. We use the built-in GPU accelerated sorting implementation in PyTorch \citep{paszke2017automatic}  in our experiments. We find that the additional computational cost added by the GroupSort activation is dwarfed by the other components of network training and inference.  

Note that MaxMin (GroupSort with a group size of 2) can be implemented either by concatenating the results of a Maxout and Minout operations (in which case it is roughly twice as costly as a single MaxOut operation), or as in its own custom CUDA  kernel \citep{nvidia2010programming}, in which case it can be as efficient as the ReLU operation.

\section{Implementing norm constraints}

\subsection{Comparing Bj\"orck and Parseval}\label{sec:bjorck_vs_parseval}

In \citet{cisse2017parseval}, the authors motivate an update to the weight matrices by considering the gradient of a regularization term, $\frac{\beta}{2}||W^T W - I||_F^2$. By subtracting this gradient from the weight matrices they push them closer to the Stiefel manifold. The final update is given by,

\begin{equation}
    W \leftarrow W (I + \beta) - \beta W W^T W
\end{equation}

Note that when $\beta=0.5$ this update is exactly the first order ($p=1$) update from Equation~\ref{eqn:bjorck_update}, with a single iteration. Compared to our approach, the key difference in Parseval networks is that the weight matrix update is applied \emph{after} the primary gradient update. Instead, we utilize Equation~\ref{eqn:bjorck_update} during forward pass to optimize directly on the Stiefel manifold. This is more expensive but guarantees that the weight matrices are close to orthonormal during training.

\textbf{Choice of $\beta$} We can relate the first order Bj\"orck algorithm to the Parseval update by setting $\beta=0.5$. However, in practice Parseval networks are trained with very small choices of $\beta$, for example $\beta=0.0003$. When $\beta$ is small the algorithm still converges to an orthonormal matrix but much more slowly. Figure~\ref{fig:bjorck_vs_beta} shows the maximum and minimum singular values of matrices which have undergone 50 iterations of the first order Bj\"orck scheme for varying choices of $\beta<0.5$. When $\beta$ is much smaller than $0.5$ the matrices may be far from orthonormal. We also show how the maximum and minimum singular values vary over the number of iterations when $\beta=0.0003$ (a common choice for Parseval networks) in Figure~\ref{fig:bjorck_vs_iters}. This has practical implications for Parseval training, particularly when using early stopping, as the weight matrices may be far from orthonormal if the gradients are relatively large compared to the update produced by the Bj\"orck algorithm. We observed this effect empirically in our MNIST classification experiments but found that Parseval networks were still able to achieve a meaningful regularization effect.

\begin{figure}[t!]
    \centering
    \includegraphics[width=0.72\linewidth]{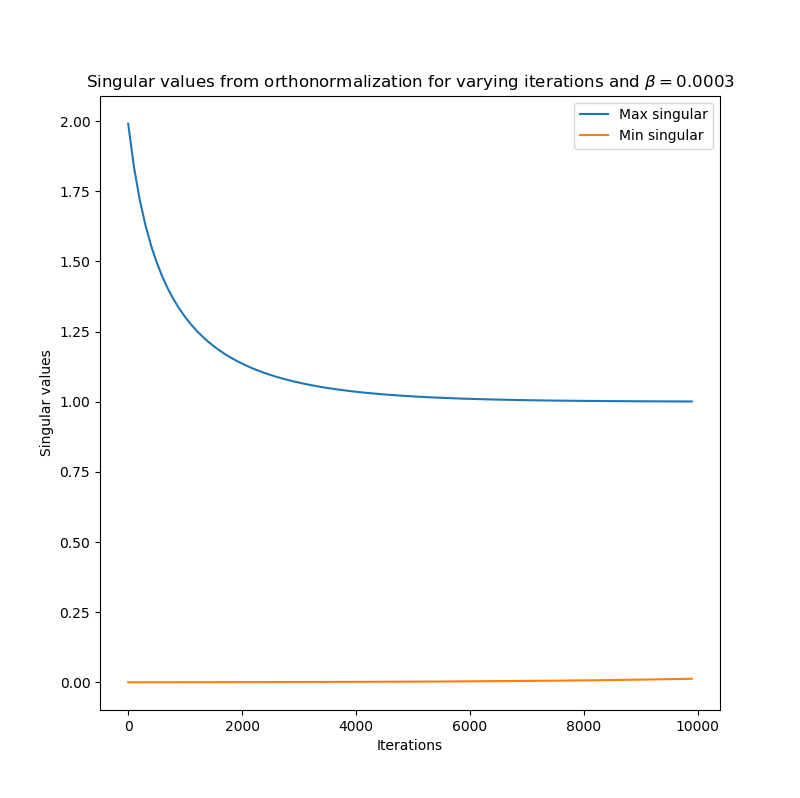}
    \caption{Convergence of the Bj\"orck algorithm for increasing iterations with $\beta=0.0003$. The largest and smallest singular values are shown after each iteration. }
    \label{fig:bjorck_vs_iters}
    \centering
    \includegraphics[width=0.72\linewidth]{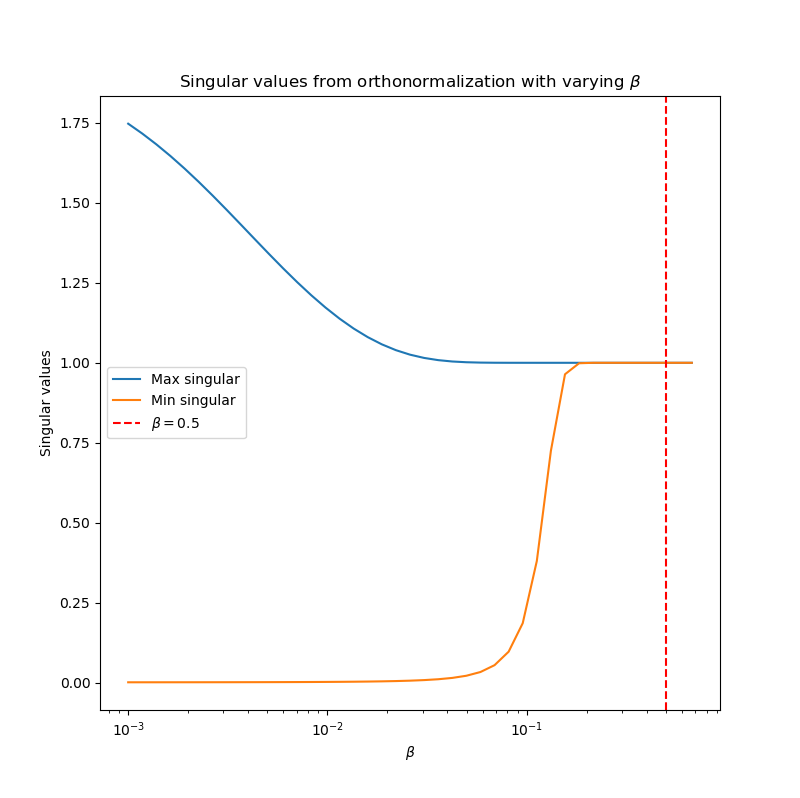}
    \caption{Convergence of the Bj\"orck algorithm for different choices of $\beta$. The largest and smallest singular values are shown after 50 iterations of the algorithm.}
    \label{fig:bjorck_vs_beta}
    \vspace{-0.3cm}
\end{figure}



\subsection{Comparing Bj\"orck and Spectral Normalization}
\label{sec:bjorck_vs_spectral}
Spectral Normalization \citep{miyato2018spectral} enforces the largest singular value of each weight matrix to be less than 1 by estimating the largest singular value and left/right singular vectors using power iteration, and normalizing the weight matrix using these during each forward pass. While this constraint does allow \textit{all} singular values of the weight matrix to be 1, we have found that this rarely happens in practice. Hence, enforcing the 1-Lipschitz constraint via spectral normalization doesn't guarantee gradient norm preservation.

We demonstrate the practical consequences of the inability of spectral normalization to preserve gradient norm on the task of approximating high dimensional cones. In order to quantify approximation performance, we carefully pick two $n$ dimensional probability distributions such that 1) The Wasserstein Distance between them is exactly 1 and 2) the optimal dual surface consists of an $n-1$ dimensional cone with a gradient of 1 everywhere, embedded in $n$ dimensions. We later train 1-Lipschitz constrained neural networks to optimize the dual Wasserstein objective in Equation \ref{eq:dual_obj} and check how well the choice of architecture is able to approximate the dual surface. Architectures that can obtain tighter estimates of Wasserstein distance are more expressive. 

Figure \ref{fig:bjorck_vs_spectral} shows that neural networks trained with Bj\"orck orthonormalization not only are able to approximate high dimensional cones better than spectral normalization, but also converge much faster in terms of training iterations. The gap between these methods gets much more significant as the problem dimensionality increases. In this experiment, each network consisted of 3 hidden layers with 512 hidden units per layer, and was trained with the Adam optimizer \citep{kingma2014adam} with its default hyperparameters. Tuned learning rates of 0.01 for Bj\"orck and 0.0033 for spectral normalization were used. 

\begin{figure}
    \centering
    \includegraphics[width=0.79\linewidth]{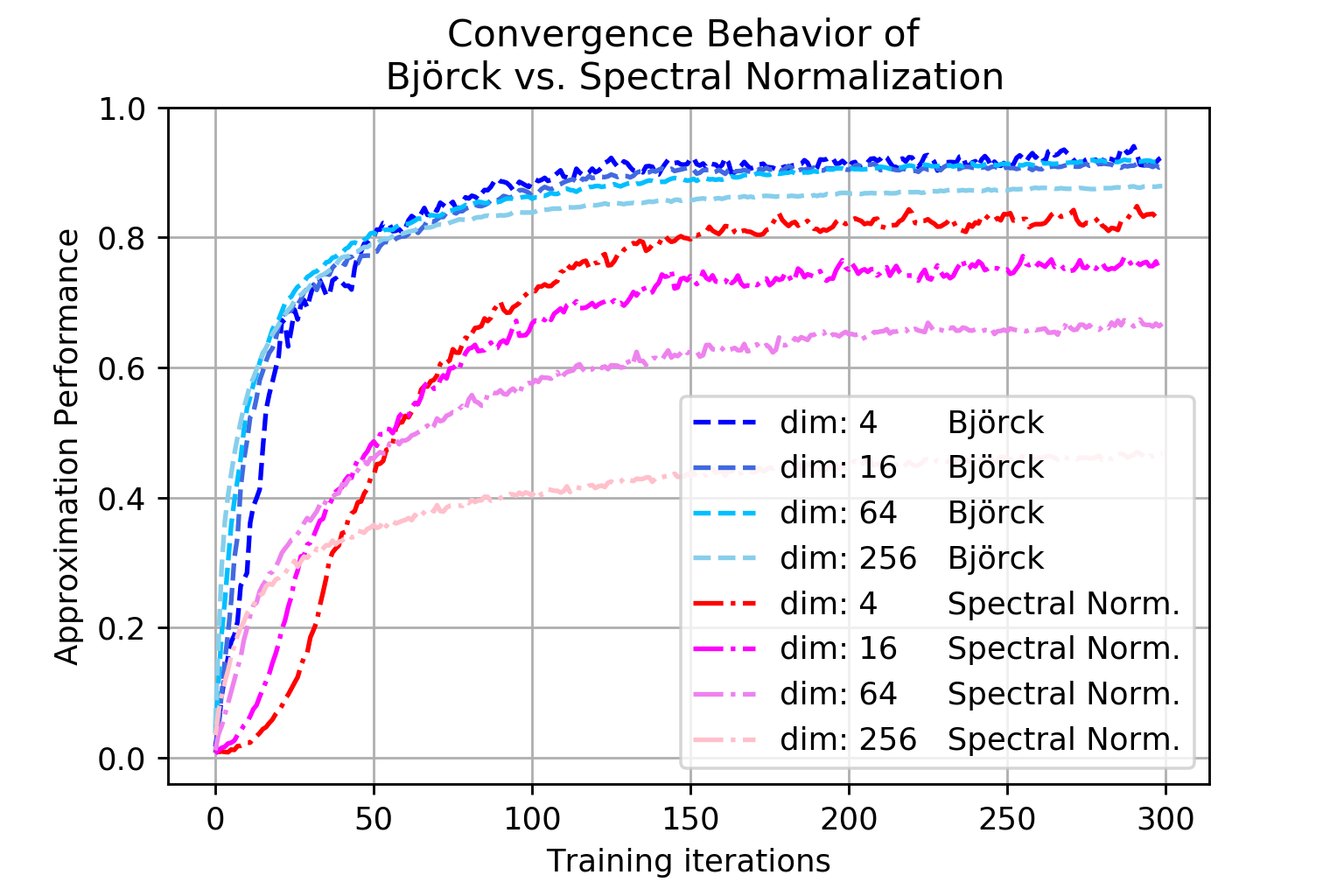}
    \caption{Comparing the performance of 1-Lipschitz neural nets using Bj\"orck orthonormalization vs. spectral normalization on the high dimensional cone fitting task (Section~\ref{sec:optimal_dual}). Using Bj\"orck orthonormalization leads to fasted convergence and better approximation performance, measured by the estimated Wasserstein Distance. }
    \label{fig:bjorck_vs_spectral}
\end{figure}

\subsection{Sufficient Condition for Convergence of Bj\"orck Orthonormalization }
\label{sec:bjorck__conv}
The Bj\"orck orthonormalization can be shown to always converge as long as the condition $||\mathbf{W}^{T}\mathbf{W} - \mathbf{I}||_{2} < 1$ is satisfied \citep{hasenclevervariational}. Since Bj\"orck orthonormalization is scale invariant, ($\mathbf{BJORCK}(\alpha \mathbf{W}) = \alpha \mathbf{BJORCK}(\mathbf{W})$) \citep{bjorck1971iterative}, the aforementioned sufficient condition can be implemented by simply scaling the weight matrix so that all of its singular values are less than or equal to 1 before orthonormalization. 

A scaling factor can be computed efficiently by considering the following matrix norm inequalities: 
\begin{align}
    \mathbf{\sigma}_{max} &\leq \sqrt{m*n}||\mathbf{W}||_{max} \label{app:max_norm} \\
    \mathbf{\sigma}_{max} &\leq \sqrt{n}||\mathbf{W}||_{1} \label{app:one_norm}\\ 
    \mathbf{\sigma}_{max} &\leq \sqrt{m}||\mathbf{W}||_{\infty} \label{app:inf_norm}
\end{align}
Above, $\mathbf{\sigma}_{max}$ corresponds to the largest singular value of the matrix and $m$ and $n$ stand for the number of rows and columns respectively. Note that computing the quantities on the right hand side of the inequalities involves at most summing over the rows or columns of the weight matrix, which is a cheap operation. 

\subsection{Computational Considerations Regarding Bjo\"rck Orthonormalization}
\label{app:computational_bjorck}
Bj\"orck orthonormalization is a costly operation even when implemented on a GPU, as it contains matrix-matrix products. In this section, we go over a few methods that can be used to accelerate Bj\"orck orthonormalization. Note that GroupSort's additional cost is \textbf{only incurred during training}: once the network is trained, it is possible to use the orthonormalized parameters as the network weights and bypass the orthonormalization step. 

\textbf{Enforcing a soft Lipschitz constraint throughout training:} We found in our experiments that one can run only a few iterations of Bj\"orck orthonormalization during training, then increase the number of iterations towards the end of training without hurting performance in classification tasks. 

\textbf{Performing spectral normalization before Bj\"orck orthonormalization:} By normalizing the weight matrices by their spectral norm before Bj\"orck orthonormalization, one can not only guarantee convergence (as described in \ref{sec:bjorck__conv}), but also faster convergence. As opposed to guaranteeing convergence by normalizing the weights using estimates of other norms (as in equations \ref{app:max_norm}, \ref{app:one_norm} and \ref{app:inf_norm}), normalizing by the spectral norm ensures the singular values of the matrix are closer to unity before Bj\"orck orthonormalization is run.

\textbf{Implementing Bj\"orck using Matrix-Vector products: } It is possible to rewrite the Bj\"orck algorithm in terms of Matrix-Vector products. This stems from the fact that we do not actually need to compute the entire orthonormal matrix $\tilde{A}$, but only a matrix-vector product $\tilde{A}v$, where $v$ are the activations of the previous network layer. Recall the expression for one Bj\"orck update:

\begin{equation*}
    A_{k+1} v = \frac{3}{2}A_{k} v - \frac{1}{2}  A_{k} A_{k}^T A_{k} v
\end{equation*}

Unfortunately, we cannot compute $A_{k+1}v$ from only $A_{k}v$ as we also need to compute $A_{k} A_{k}^T A_{k} v$ which requires $A_{k}$ explicitly. However, we can rewrite the above using two operations: $u \mapsto A_{k} u$ and $u \mapsto A_{k} A_{k}^T u$. To see why this is useful, we write,

\begin{align*}
    A_{k+1} A_{k+1}^T v &= (\frac{3}{2}A_{k} - \frac{1}{2} A_{k} A_{k}^T A_{k}) \\ &(\frac{3}{2}A_{k}^T - \frac{1}{2} A_{k}^T A_{k} A_{k}^T)v \\
    &= \frac{9}{4} A_{k}A_{k}^T v -\frac{3}{2} (A_{k} A_{k}^T)(A_{k} A_{k}^T)v \\
    & + \frac{1}{4} (A_{k} A_{k}^T)(A_{k} A_{k}^T)(A_{k} A_{k}^T)v
\end{align*}

Hence, we can write $u \mapsto A_{k+1}A_{k+1}u$ as a function of $u \mapsto A_{k}A_{k}u$. This allows us to recursively define the matrix-vector product of the $k$th iterate in terms of the previous iterates.

This method works very well for relatively few iterations (approximately less than 5) but scales poorly as the number of iterations increases. This is because the algorithm requires $O(3^k)$ matrix-vector products. Table~\ref{tab:bjorck_runtime} shows a runtime comparison for the original algorithm and the Matrix-Vector Product (MVP) for increasing iterations averaged over 10 runs. The weight matrices have a dimension of $1000 \times 1000$ and are normalized using the equation \ref{app:inf_norm} to guarantee convergence. 

\begin{table}
\centering
\resizebox{\linewidth}{!}{
 \begin{tabular}{|lrrrrr|} 
 \hline
 Iterations & 1 & 2 & 3 & 5 & 10 \\ \hline
 Full Bj\"orck & 0.020 & 0.039 & 0.059 & 0.095 & 0.21 \\ 
 MVP Bj\"orck  & 0.0002 & 0.0005 & 0.0011 & 0.012 & 2.88 \\ \hline
 Speedup Factor & 99.98x & 78.59x & 53.56x & 7.77x & 0.07x \\
 \hline
 \end{tabular}}
 \caption{Runtime (seconds) for full Bj\"orck and Matrix-Vector Product (MVP) Bj\"orck for a $1000\times 1000$ matrix, averaged over 10 runs. }
 \label{tab:bjorck_runtime}
\end{table}

\section{Projecting Vectors on $L_{\infty}$ Ball}
\label{app:proj_linf}
The following algorithm uses sorting to project vectors on $L_{\infty}$ balls \citep{condat2016fast}. 

\begin{algorithm}
  \caption{$L_{\infty}$ Projection via. Sorting}\label{alg:spec_jac}
  \begin{algorithmic}
  \STATE Input: $\vy \in \R^N $, Output: $\vx \in \R^N $ . \\
  \STATE Sort $\vy$ into $\vu$: $u_1 \geq \ldots \geq u_N $ . \\
  \STATE Set $K := \max_{1 \leq k \leq N}\{k | (\sum_{r=1}^{k} u_r - 1) / k < u_k \}$. \\
  \STATE Set $\tau := (\sum_{k}^{K}u_k - 1)/K$. \\
  \FOR{$n = 1, \ldots , N$}
    \STATE Set $x_n := \max_{y_n - \tau,0}$ \\
  \ENDFOR
    \end{algorithmic}
\end{algorithm}
\section{Non-expressive norm-constrained networks are linear}
\label{app:repr_thm}

\reprproof*

\begin{proof}
We can express the input-output Jacobian of a neural network as: 
\[
\frac{\partial f}{\partial \rvx} = \frac{\partial f}{\partial \vh_{L-1}}\frac{\partial \vh_{L-1}}{\partial \vz_{L-1}}\frac{\partial \vz_{L-1}}{\partial \rvx} = \mathbf{W}_L \frac{\partial \mathbf{\phi}(\vz_{L-1})}{\partial \vz_{L-1}}\frac{\partial \vz_{L-1}}{\partial \rvx}
\]
Note that $\mathbf{W}_L \in \mathbb{R}^{1 \times n_{L-1}}$. Moreover, using the sub-multiplicativity of matrix norms, we can write:
\begin{equation*}
\begin{aligned}
1 = \leftpar \frac{\partial f}{\partial \rvx} \rightpar_2 &\leq \leftpar \mathbf{W}_L \frac{\partial \mathbf{\phi}(\vz_{L-1})}{\partial \vz_{L-1}} \rightpar_2 \leftpar\frac{\partial \vz_{L-1}}{\partial \rvx} \rightpar_2 \\
&\leq \leftpar \mathbf{W}_L \rightpar_2 \leftpar \frac{\partial \mathbf{\phi}(\vz_{L-1})}{\partial \vz_{L-1}} \rightpar_2 \leftpar \frac{\partial \vz_{L-1}}{\partial \rvx} \rightpar_2 \leq 1
\end{aligned}
\end{equation*}
for $x$ almost everywhere. The quantity is also upper bounded by 1 due to the 1-Lipschitz property. Therefore, all of the Jacobian norms in the above equation must be equal to 1. Notably,
\[
\leftpar \mathbf{W}_L \frac{\partial \mathbf{\phi}(\vz_{L-1})}{\partial \vz_{L-1}} \rightpar_2 = 1 \quad \textrm{and} \quad \leftpar\mathbf{W}_{L} \rightpar_2 = 1 
\]
We then consider the following operation: 
\begin{equation}
\begin{aligned}
&\leftpar \mathbf{W}_{L} \rightpar_2^{2} - \leftpar \mathbf{W}_L \frac{\partial \mathbf{\phi}(\vz_{L-1})}{\partial \vz_{L-1}} \rightpar^2_2 \\
&= \sum_{i=1}^{n} (1 - \Big(\frac{\partial \mathbf{\phi}(\vz_{L-1})}{\partial \vz_{L-1}}\Big)^{2}_{ii})(W_{L,i})^{2} = 0
\end{aligned}
\end{equation}
\label{key_ineq}
We have $0 \leq \frac{\partial \phi}{\partial \vz_L} \leq 1$ as $\phi$ is 1-Lipschitz and monotonically increasing. Therefore, we must have either $\frac{\partial \phi}{\partial \vz_L}_{ii} = 1$ almost everywhere, or $\mathbf{W}_{L,i} = 0$. Thus we can write,
\begin{equation*}
\begin{aligned}
\vz_{L} &= \sum_{i=1}^{m} W_{L,i} \phi(\vz_{L-1})_i + b_{L} \\
&= \sum_{i: W_{L,i} \neq 0} W_{L,i} \phi(\vz_{L-1})_i + b_{L} \\
&= \sum_{i: W_{L,i} \neq 0} W_{L,i} \vz_{L-1, i} + b_{L}
\end{aligned}
\end{equation*}
Then $\vz_L$ can be written as a linear function of $\vz_{L-1}$ almost everywhere and by Lipschitz continuity we must have that $\vz_L$ \emph{is} a linear function of $\vz_{L-1}$. In particular, we can write $\vz_{L} = \mathbf{W}_L \mathbf{W}_{L-1} \vh_{L-2} + (\mathbf{W}_L \vb_{L-1} + b_{L})$, collapsing the last two layers into a single linear layer, with weight matrix $\mathbf{W}_L \mathbf{W}_{L-1} \in \mathbb{R}^{1\times n_{L-2}}$ and scalar bias $\mathbf{W}_L \vb_{L-1} + b_L$.

From here we can apply the exact same argument as above to $\phi(\rvz_{L-2})$, reducing the next layer to be linear. By repeating this all the way to the first linear layer we collapse the network into a single linear function.
\end{proof}
\equivweights*
\begin{proof}
Take a weight matrix $\mathbf{W}_i$, for $i < L$. By the argument in the proof of Theorem~\ref{repr_proof}, this matrix must preserve the norm of gradients during backpropagation. That is,
\[
1 =  \leftpar \frac{\partial f}{\partial \vz_{i}} \mathbf{W}_i \rightpar_2
\]
Using the singular value decomposition, we write $\mathbf{W}_i = \mathbf{U}\mathbf{\Sigma}\mathbf{V}^T$. We then define $\tilde{\mathbf{W}}_i = \mathbf{U}\tilde{\mathbf{\Sigma}}\mathbf{V}^T$ where $\tilde{\mathbf{\Sigma}}$ has ones along the diagonal. Furthermore, define $\mathbf{W}^{(t)}_i = t \mathbf{W}_i + (1-t) \tilde{\mathbf{W}}_i$. Replacing $\mathbf{W}_i$ with $\mathbf{W}^{(t)}_i$ in the network:
\begin{equation*}
\begin{aligned}
\frac{\partial f}{\partial t} &= \frac{\partial f}{\partial \vz_i} \frac{\partial \vz_i}{\partial t}
= \frac{\partial f}{\partial \vz_i} (\mathbf{W}_i - \tilde{\mathbf{W}}_i) \vh_{i-1} \\ &= \frac{\partial f}{\partial \vz_i} \mathbf{U}(\mathbf{\Sigma}_i - \tilde{\mathbf{\Sigma}}_i)\mathbf{V}^T \vh_{i-1}
\end{aligned}
\end{equation*}
As the norm of $\frac{\partial f}{\partial \vz_i}$ is preserved by $\mathbf{W}_i$ we must have that $\vu = (\frac{\partial f}{\partial \vz_i}\mathbf{U})^T$ has non-zero entries only where the diagonal of $\mathbf{\Sigma}$ is 1. That is, $u_j = 0 \iff \mathbf{\Sigma}_{jj} < 1$. In particular, we have $\vu^T \mathbf{\Sigma}_i = \vu^T \tilde{\mathbf{\Sigma}}_i$ meaning $\frac{\partial f}{\partial t} = 0$. Thus, the output of the network is the same for all $t$, in particular for $t=0$ and $t=1$. Thus, we can replace $\mathbf{W}_i$ with $\tilde{\mathbf{W}}_i$ and the network output remains unchanged.

We can repeat this argument for all $i < L$ (for $i=1$ we adopt the notation $\vh_0 = \vx$, the input to the network). For $i=L$ the result follows directly.
\end{proof}
\section{Universal Approximation of 1-Lipschitz Functions}
\label{app:univ}

Here we present formal proofs related to finding neural network architectures which are able to approximate any 1-Lipschitz function. We begin with a proof of \Lemref{lemma:stoneweierstrass}.
\medskip
\swtheorem*

\begin{proof}
This proof follows a standard approach with small modifications. We aim to show that for any $g \in \lfuncs$ and $\epsilon > 0$ we can find $f \in L$ such that $||g-f||_{\infty} < \epsilon$ (i.e. the largest difference is $\epsilon$).

Fix $x \in X$. Then for each $y \in X$, we have an $f_{y} \in L$ with $f_{y}(x) = g(x)$ and $f_{y}(y) = g(y)$. This follows from the separation property of $L$ and, using the fact that $g$ is 1-Lipschitz, $|g(x) - g(y)| \leq d_{X}(x,y)$.

Define $V_{y} = \{ z \in X: f_{y}(z) < g(z) + \epsilon \}$. Then $V_{y}$ is open and we have $x,y \in V_{y}$. Therefore, the collection of sets $\{ V_{y}\}_{y\in X}$ is an open cover of $X$. By the compactness of $X$, there exists some finite subcover of $X$, say, $\{ V_{y_1}, \ldots , V_{y_n}\}$, with corresponding functions $f_{y_1}, \ldots ,f_{y_n}$.

Let $F_{x} = min(f_{y_1}, \ldots , f_{y_n})$. Since $L$ is a lattice we must have $F_{x} \in L$. And moreover, we have that $F_{x}(x) = g(x)$ and $F_{x}(z) < g(z) + \epsilon$, for all $z \in X$.

Now, define $U_{x} = \{z \in X: F_{x}(z) > g(z) - \epsilon \}$. Then $U_{x}$ is an open set containing $x$. Therefore, the collection $\{ U_{x} \}_{x \in X}$ is an open cover of $X$ and admits a finite subcover, $\{ U_{x_1}, \ldots , U_{x_m}\}$, with corresponding functions $F_{x_1}, \ldots , F_{x_m}$.

Let $G = max(F_{x_1}, \ldots , F_{x_m}) \in L$. We have $G(z) > g(z) - \epsilon$, for all $z \in X$.

Combining both inequalities, we have that $g(z) - \epsilon < G(z) < g(z) + \epsilon$, for all $z \in X$. Or more succinctly, $||g-G||_{\infty} < \epsilon$. The result is proved by taking $f=G$.
\end{proof}
We now proceed to prove \Thmref{thm:universalapprox}.
\bigskip
\lpuniv*
\begin{proof}
The first property we require is separation of points. This follows trivially as given four points satisfying the required conditions we can find a linear map with the required $L_{p,\infty}$ matrix norm that fits them. It remains then to prove that we can construct a lattice under this constraint. We begin by considering two 1-Lipschitz neural networks, $f$ and $g$. We wish to design an architecture which is guaranteed to be 1-Lipschitz and can represent $\max(f,g)$ and $\min(f,g)$.

The key insight is that we can split the network into two parallel \emph{channels} each of which computes one of $f$ and $g$. At the end of the network, we can then select one of these channels depending on whether we want the max or the min.

Each of the networks $f$ and $g$ is determined by a set of weights and biases, we will denote these $[\rmW^{f}_{1}, \rvb^{f}_{1}, \ldots, \rmW^{f}_{n}, b^{f}_{n}]$ and $[\rmW^{g}_{1}, \rvb^{g}_{1}, \ldots,  \rmW^{g}_{n}, \rvb^{g}_{n}]$ for $f$ and $g$ respectively. For now, assume that these networks are of equal depth (we can lift this assumption later) however we make no assumptions on the width. We will now construct $h = max(f,g)$ in the form of a 1-Lipschitz neural network. We will design a network $h$ which first concatenates the first layers of networks $f$ and $g$ and then computes $f$ and $g$ separately before combining them at the end.

We take the first weight matrix of $h$ to be $\rmW^{h}_{1} = [\rmW^{f}_{1} \ \rmW^{g}_{1}]^T$, the weight matrices of $f$ and $g$ stacked vertically. This matrix necessarily satisfies $||\rmW^{h}_{1}||_{p, \infty} = 1$. Similarly, the bias will be those from the first layers of $f$ and $g$ stacked vertically. Then the first layer's pre-activations will be exactly the pre-activations of $f$ and $g$ stacked vertically.

For the following layers, we construct the biases in the same manner (vertical stacking). We construct the weights by constructing new block-diagonal weight matrices. That is, given $\rmW_{i}^{f}$ and $\rmW_{i}^{g}$, we take
\[ W_{i}^{h} = \left[\begin{array}{cc}

W_{i}^{f} \quad 0 \\
0 \quad W_{i}^{g}
\end{array}\right]\]

This matrix also has $\infty$-norm equal to 1. We repeat this for each of the layers in $f$ and $g$ and end up with a final layer which has two units, $f$ and $g$. We can then take MaxMin of this final layer and take the inner product with $[1,0]$ to recover the max or $[0, 1]$ for the min.

Finally, we must address the case where the depth of $f$ and $g$ are different. In this case we notice that we are able to represent the identity function with MaxMin activations. To do so observe that after the pre-activations have been sorted we can multiply by the identity and the sorting activation afterwards will have no additional effect. Therefore, for the channel that has the smallest depth we can add in these additional identity layers to match the depths and resort to the above case.

We have shown that the set of neural networks is a lattice which separates points, and thus by \Lemref{lemma:stoneweierstrass} it must be dense in $\lfuncs$.
\end{proof}

Note that we could have also used the maxout activation \cite{pmlr-v28-goodfellow13} to complete this proof. This makes sense, as the maxout activation is also norm-preserving in $L_\infty$. However, this does not hold when using a 2-norm constraint on the weights. We now present several consequences of the theoretical results given above.

This result can be extended easily to vector-valued Lipschitz functions with respect to $L_\infty$ distance by noticing that the space of such 1-Lipschitz functions is a lattice. We may apply the Stone-Weierstrass proof to each of the coordinate functions independently and use the same construction as in \Thmref{thm:universalapprox} modifying only the last layer which will now reorder the outputs of each function to do a pairwise comparison and then select the relevant components to produce the max or the min.

\linfuniv*

\begin{proof}
Note that given two functions, $g,f: \reals^n \rightarrow \reals^m$ which are 1-Lipschitz with respect to the $L_\infty$ metric, their element-wise max (or min) is also 1-Lipschitz with respect to the $L_\infty$ metric. Consider the element-wise components of such an $f$, written $f = (f_1,\ldots,f_m)$. We can apply the Stone-Weierstrass theorem (\Lemref{lemma:stoneweierstrass}) to each of the components independently, such that if the same conditions apply (trivially extended to $\reals^m$) the Lattice is dense. Thus, as in the proof of \Thmref{thm:universalapprox}, it suffices to find a network $h \in \mathcal{LN}^{m}_{\infty}$ which can represent the max or min of any other networks, $f,g \in \mathcal{LN}^{m}_{\infty}$.

In fact, we can use almost exactly the same construction as in the proof of \Thmref{thm:universalapprox}. We follow the same initial steps by concatenating weight matrices and constructing block-diagonal matrices from the two networks. After doing this for all layers in the networks $f$ and $g$, we will output $[f_1, \ldots, f_m, g_1, \ldots g_m$]. We can then permute these entries using a single linear layer to  produce $[f_1, g_1, f_2, g_2, \ldots, f_m, g_m]$ finally we take MaxMin and use the final weight matrix to select either $\max(f,g)$ or $\min(f,g)$.
\end{proof}

\section{Spectral Jacobian Regularization}\label{sec:jac_reg}

Most existing work begins with the goal of constraining the spectral norm of the Jacobian and proceeds to achieve this by placing constraints on the weights of the network \citep{yoshida2017spectral}. While not the main focus of our work, we propose a simple new technique which allows us to directly regularize the spectral norm of the Jacobian, $\sigma(J)$. This method differs from the ones described previously as the Lipschitz constant of the entire network is regularized using a single term, instead of at the layer level.

The intuition for this algorithm follows that of \citet{yoshida2017spectral}, who apply power iteration to estimate the singular values of the weight matrices online. The authors also discuss computing the spectral radius of the Jacobian directly, and related quantities such as the Frobenius norm, but dismiss this as being too computationally expensive.

Power iteration can be used to compute the leading singular value of a matrix $J$ with the following repeated steps,
\[ \rvv_{k} = J^T \rvu_{k-1} / ||J^{T} \rvu_{k-1}||_2, \rvu_{k} = J\rvv_{k} / ||J\rvv_{k}\|_2 \]
Then we have $\sigma(J) \approx \rvu^{T} J \rvv$. There are two challenges that must be overcome to implement this in practice. First, the algorithm requires higher order derivatives which leads to increased computational overhead. However, the tradeoff is often reasonable in practice, see e.g. \citet{drucker1992improving}. Second, the algorithm requires both Vector-Jacobian products and Jacobian-Vector products. The former can be computed with reverse-mode automatic differentiation but the latter requires the less common forward-mode. Fortunately, one can recover forward-mode from reverse mode by constructing Vector-Jacobian products and utilizing the transpose operator \citep{townsendTrick}. We can re-use the intermediate reverse-mode backpropagation within the algorithm which further reduces the computational overhead. The algorithm itself is presented as Algorithm \ref{alg:spec_jac}.

\begin{algorithm}
  \caption{Spectral Jacobian Regularization}\label{alg:spec_jac}
  \begin{algorithmic}
  \STATE Initialize $u$ randomly, choose hyperparameter $\lambda > 0$ \\
  \FOR{data batch (X,Y)}
    \STATE Compute logits $f_\theta(X)$ \\
    \STATE Compute loss $\mathcal{L}(f_\theta(X),Y)$ \\
    \STATE Compute $\rvg = \rvu^T \dfrac{\partial \rvf}{\partial \rvx}$, using reverse mode \\
    \STATE Set $\rvv =  \rvg/ ||\rvg||_2$ \\
    \STATE Compute $\rvh = (\rvv^T \dfrac{\partial \rvg}{\partial \rvu})^T = \dfrac{\partial \rvf}{\partial \rvx} \rvv$, using reverse mode \\
    \STATE Update $\rvu = \rvh/||\rvh||_2$ \\
    \STATE Compute parameter update from $\dfrac{\partial}{\partial \theta}\left (\mathcal{L} + \lambda \rvu^T \rvh \right  )$
    \ENDFOR
    \end{algorithmic}
\end{algorithm}

We present this algorithm primarily to be used for regularization but this could also be used to approximately control the Lipschitz constraint by rescaling the output of the entire network by the estimate of the Jacobian spectral norm similar to spectral normalization \cite{miyato2018spectral}. 

\section{Additional Experiments}\label{sec:additional_exp}

We present additional experimental results. 

\subsection{Classification}\label{sec:app_classification}

We compared a wide range of Lipschitz architectures and training schemes on some simple benchmark classification tasks. We demonstrate that we are able to learn Lipschitz neural networks which are expressive enough to perform classification without sacrificing performance.

\paragraph{MNIST Classification} We explored classification with a 3-layer fully connected network with 1024 hidden units in each layer. Each model was trained with the Adam optimizer \citep{kingma2014adam}. The results are presented in Table. \ref{tab:mnist_classify}.

For all models the GroupSort activation is able to perform classification well - especially when the Lipschitz constraint is enforced. Surprisingly, we found that we could apply the GroupSort activation to sort the entire hidden layer and still achieve reasonable classification performance, even with dropout.
In terms of classification performance, spectral Jacobian regularization was most effective (Appendix~\ref{sec:jac_reg}).

While the Parseval networks are capable of learning a strict Lipschitz constraint this does not always hold in practice. A small beta value leads to slow convergence towards orthonormal weights. When early stopping is used, which is typically important for good validation accuracy, it is difficult to ensure that the resulting network is 1-Lipschitz. 

\begin{table}
    \centering
    \resizebox{\linewidth}{!}{
    \begin{tabular}{|l|r|r|r|r|r|}
    \hline
    & ReLU & MaxMin & GS(4) & FullSort & Maxout \\ \hline
    Standard & 1.61 & 1.47 & 1.62 & 3.53 & 1.40 \\
    Dropout & 1.27 & 1.37 & 1.29 & 3.62 & 1.27 \\
    Bj\"orck & 1.54 & 1.25  & 1.43 & 2.06 & 1.43 \\
    Spectral Norm & 1.54 & 1.26 & 1.32 & 2.94 & 1.26 \\
    Spectral Jac & 1.05 & 1.09 & 1.24 & 1.93 & 1.02 \\
    Parseval & 1.43 & 1.40 & 1.44 & 3.36 & 1.35 \\
    $L_\infty$ & 2.25 & 2.28 & 2.22 & 4.88 & 1.98 \\ \hline
    \end{tabular}}
    \caption{\textbf{MNIST classification} Test error shown for different architectures and activations (GS stands for GroupSort.).}\label{tab:mnist_classify} 
\end{table}%

\paragraph{Classification with little data} While enforcing the Lipschitz constraint aggressively could hurt overall predictive performance, it decreases the generalization gap substantially. Motivated by the observations of \citet{bruna2013scattering} we investigated the performance of Lipschitz networks on small amounts of training data, where learning robust features to avoid overfitting is critical.

For these experiments we kept the same network architecture as before. We trained standard unregularized networks, networks with dropout, networks regularized with weight decay, and 1-Lipschitz neural networks enforced with the Bj\"orck algorithm. We made use of a LeNet-5 architecture, with convolutions and max-pooling --- the latter prevents norm preservation and thus may reduce the effectiveness of MaxMin substantially. We found that Dropout was the most effective regularizer in this case but confirmed that networks with Lipschitz constraints were able to significantly improve generalization. Full results are in Table~\ref{tab:small_mnist_classify}.

\begin{table*}[]
    \centering
    \begin{tabular}{|l|r|r|r|r|r|r|r|r|}
    \hline
    \multirow{2}{*}{Data Size} & \multicolumn{2}{|c|}{Standard} & \multicolumn{2}{|c|}{Dropout} & \multicolumn{2}{|c|}{Weight Decay} & \multicolumn{2}{|c|}{Bj\"orck} \\ \cline{2-9}
    & ReLU & MaxMin & ReLU & MaxMin & ReLU & MaxMin & ReLU & MaxMin  \\ \hline
    300 & 12.40 & 12.14 & 7.30 & 10.64 & 11.06 & 10.81 & 8.12 & 7.81 \\
    500 & 8.57 & 9.13 & 5.54 & 6.15 & 7.33 & 7.50 & 5.96 & 6.98 \\
    1000 & 5.95 & 6.23 & 3.70 & 4.58 & 5.14 & 6.05 & 4.45 & 4.54 \\
    5000 & 2.54 & 2.51 & 1.84 & 2.15 & 2.31 & 2.55 & 2.23 & 2.31  \\
    10000 & 1.77 & 1.76 & 1.26 & 1.70 & 1.58 & 1.57 & 1.66 & 1.64  \\ \hline
    \end{tabular}
    \caption{\textbf{MNIST Classification with limited data} Test error for varying architectures and activations per training data size.}\label{tab:small_mnist_classify}
    \vspace{-0.3cm}
\end{table*}%

\paragraph{Classification on CIFAR-10} We briefly explored classification on CIFAR-10 using Wide ResNets (Depth 28, Width 4) \citep{Zagoruyko2016WideRN, he2016deep}. We performed these experiments primarily to explore the effectiveness of the MaxMin activation in a more challenging setting. We used the optimal optimization hyperparameters for ReLU with SGD and performed a small search over regularization parameters for Parseval and Spec Jac regularization. We present results in Table~\ref{tab:cifar_classify}. We found that MaxMin performed comparably to ReLU in this setting and hope to explore this further in future work.

\begin{table*}
    \centering
    \begin{tabular}{|l|r|r|r|r|r|r|}
    \hline
    \multirow{2}{*}{} & \multicolumn{2}{|c|}{Standard} & \multicolumn{2}{|c|}{Parseval} & \multicolumn{2}{|c|}{Spec Jac Regularization} \\ \cline{2-7}
    & ReLU & MaxMin & ReLU & MaxMin & ReLU & MaxMin  \\ \hline
    CIFAR-10 & 95.29 & 94.57 & 95.45 & 94.83 & 95.44 & 94.62 \\ \hline
    \end{tabular}
    \caption{\textbf{CIFAR-10 Classification} Test accuracy for Wide ResNets (Depth 28, Width 4) with varying activations and training schemes.}\label{tab:cifar_classify}
\end{table*}%

\begin{figure}[t!]
    \centering
    \begin{minipage}{0.58\linewidth}
      \centering
      \includegraphics[width=\linewidth]{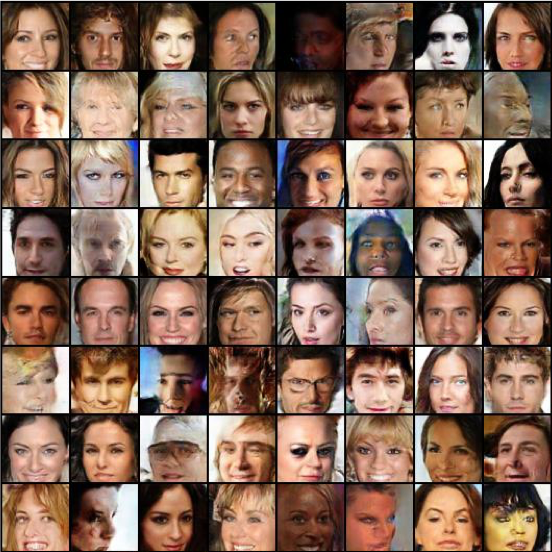}\\
      (a) Leaky-ReLU Critic
    \end{minipage}\hspace{1cm}
    \begin{minipage}{0.58\linewidth}
      \centering
      \includegraphics[width=\linewidth]{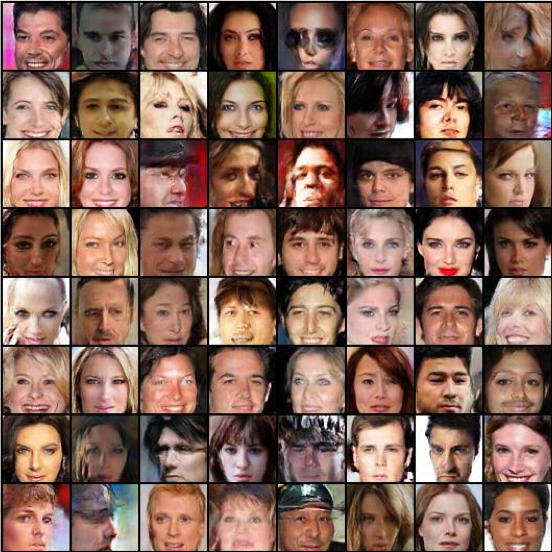}\\
      (b) MaxMin Critic
    \end{minipage}
    \caption{Generated images from WGAN-GP models trained on the CelebA dataset.}\label{fig:celeba_wgan}
\end{figure}
\vspace{-0.1cm}
\subsection{Training WGAN-GP}

We found that the MaxMin activation could also be used as a drop-in replacement for ReLU activations in WGAN architectures that utilize a gradient-norm penalty in the training objective. We took an existing implementation of WGAN-GP which used a fully convolutional critic network with 5 layers and LeakyReLU activations. The generator used a linear layer followed by 4 deconvolutional layers. We trained this model with the tuned hyperparameters for the LeakyReLU activation and then used the same settings to train a model with MaxMin acivations. We defer a more thorough study of this setting to future work but present here the output of the trained generators after 50 epochs of training on the CelebA dataset \citep{liu2015faceattributes} in Figure~\ref{fig:celeba_wgan}.

\subsection{Dynamical Isometry}\label{app:dyn_iso}

In Figure~\ref{fig:all_sn_hist} we plot the distribution of all singular values of ReLU and GroupSort 2-norm-constrained networks trained as MNIST classifiers, with a Lipschitz constant of 10. While the ReLU singular values are spread between 4-8 the GroupSort network concentrates the singular values in range 9-10. Dynamical isometry \citep{pennington2017resurrecting} requires all Jacobian singular values to be concentrated around 1. Using 2-norm constraints and GroupSort activations we are able to achieve dynamical isometry throughout training.

\begin{figure}[b!]
    \centering
    \includegraphics[width=0.8\linewidth]{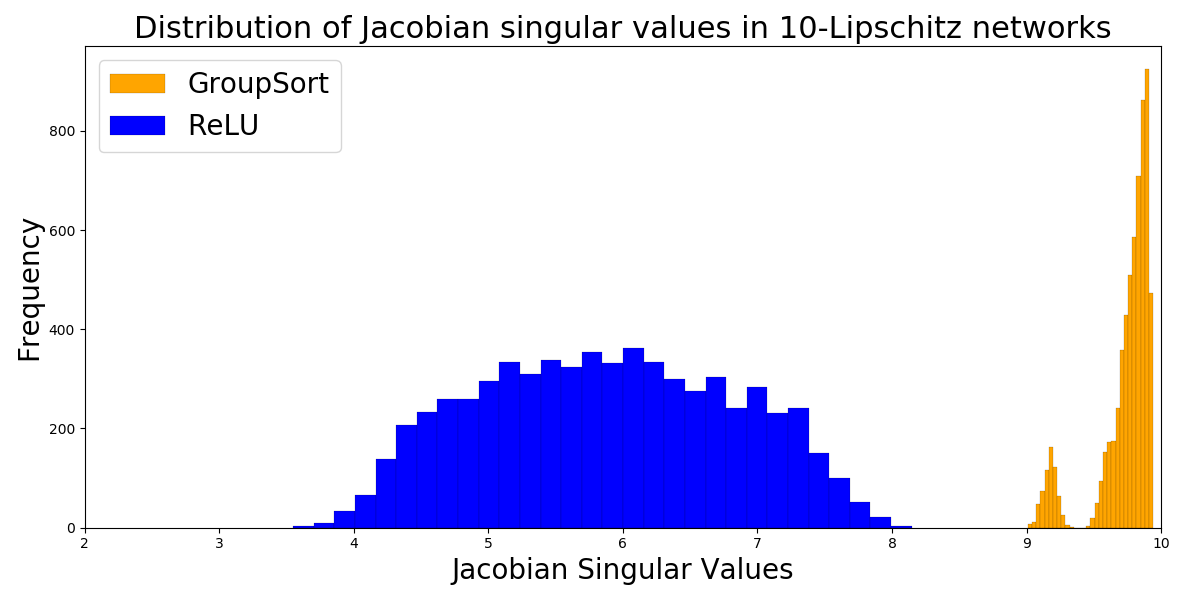}
    \caption{\textbf{Jacobian singular values distribution} We compare the Jacobian singular values of ReLU and GroupSort networks.}
    \label{fig:all_sn_hist}
\end{figure}

\section{Experiment Details}\label{app:exp_details}
We present additional experimental details. 

\subsection{Designing Synthetic Distributions for Wasserstein Distance Estimation}
\label{app:surf}
\textbf{Absolute value}
We pick $\displaystyle p_{1}(\rx) =  \delta_{0}(x)$ and $\displaystyle p_{2}(\rx) =  \frac{1}{2}\delta_{-1}(x) + \frac{1}{2}\delta_{1}(x)$, where $\delta_{\alpha}(x)$ stands for the Dirac delta function located at $\alpha$. The optimal dual surface learned while computing the Wasserstein distance between $\displaystyle p_{1}$ and $\displaystyle p_{2}$ is the absolute value function. This also makes intuitive sense, as the function that assigns "as low values as possible" at $x=0$ and assigns "as high values as possible" at $x=-1$ and $x=1$ while satistying 1-Lipschitz condition must be the absolute value function. 

The transport plan that minimizes the primal objective will simply be to map the center Dirac delta equally to the ones near it. This leads to a Wasserstein distance of 1. 

The networks we trained had 3 hidden layers each with 128 hidden units. We report the results obtained with the Aggretated Momentum optimizer (AggMo) \citep{lucas2018aggregated} with its default parameters, as it lead to faster convergence in our experiments compared to Adam optimizer \citep{kingma2014adam}. We note that the choice of optimizer had minimal impact on the final Wasserstein Distance estimates. 

\textbf{Multiple 2D Circular Cones}
 We describe the probability distributions  $\displaystyle p_{1}$ and $\displaystyle p_{2}$ implicitly by describing how we sample from them. $\displaystyle p_{1}$ is sampled from by selecting one of the three points ($(-2, 0)$, $(0, 0)$ and $(2, 0)$) uniformly. $\displaystyle p_{2}$ is sampled from by first uniformly selecting one of the three points aforementioned, then uniformly sampling a point on the circle surrounding it, with radius 1. Wasserstein dual problem aims to find a Lipschitz function which assigns "as high as possible" values to the three points, and "as low as possible" values to the circles with radius 1 surrounding the three points. Hence, the optimal dual function must consist of three cones centered around $(-2, 0)$, $(0, 0)$ and $(2, 0)$. The behavior of the function outside this support doesn't have an impact on the solution. 
 
The optimal transport plan must map the probability mass to the nearby circles surrounding them uniformly. This leads to an Wasserstein distance of 1.0. 

The networks we trained had 3 hidden layers with 312 hidden units. We used the Aggretated Momentum optimizer (AggMo) \citep{lucas2018aggregated} with its default parameters. 

\textbf{$\vn$ Dimensional Circular Cones}
This is a simple extension of the absolute value case described above. 

We pick $\displaystyle p_{1}$ as the Dirac delta function located at the origin, and sample from  $\displaystyle p_{2}$ by uniformly selecting a point from high dimensional spherical shell with radius 1, centered at the origin. Following similar arguments developed for absolute value, it can be shown that the optimal dual function is a single high dimensional circular cone and the Wasserstein distance is also equal to unity. 

\subsection{Wasserstein Distance Estimation}
\label{app:wde}
The GAN variants we trained on MNIST and CIFAR10 datasets used the WGAN formulation first introduced in \citet{arjovsky2017wasserstein}. The architectures of the generator and critic networks were the same as the ones used in\citep{chen2016infogan}. For the subsequent task of Wasserstein distance estimation, the weights of the generator networks were frozen after the initial GAN training has converged. 

\begin{figure}[b!]
\centering
  \includegraphics[width=0.58\linewidth]{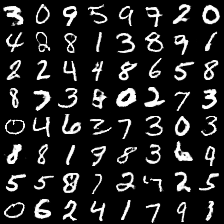}
  \includegraphics[width=0.58\linewidth]{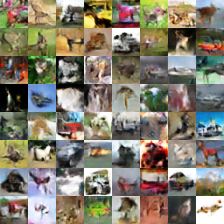}
\centering
\caption{Samples from WGANs trained on MNIST and CIFAR10 whose critics use gradient norm preserving units. }
\label{f:gan_outs}
\end{figure}

\subsection{Wasserstein GAN with 1-Lipschitz Layers}
\label{app:wgan}
We borrowed the discriminator and generator networks from \citet{chen2016infogan}, but switched the ReLU activations with MaxMin and replaced the convolutional and fully connected layers with their Bj\"orck counterparts. We didn't use batch normalization, as this would violate the Lipschitz constraint. 


\subsection{Classification}
For MNIST classification, we searched the hyperparameters as follows. For Bj\"orck, $L_\infty$ constrained, and Spectral Norm architectures we tried networks with a guaranteed Lipschitz constant of 0.1, 1, 10 or 100. For Parseval networks we tried $\beta$ values in the range 0.001, 0.01, 0.1, 0.5. For SpecJac regularization we scaled the penalty by 0.01, 0.05, or 0.1.

In order to scale the Lipschitz constant of the network, we introduce constant scaling layers in the network such that the product of the constant scale parameters is equal to the Lipschitz constant. As the activation functions are homogeneous, e.g.~$\text{ReLU}(a\rvx) = a \text{ReLU}(\rvx)$, this is equivalent to scaling the output of the network as described in Section~\ref{mthd}.

\subsection{Robustness and Interpretability}

For the adversarial robustness experiments we trained fully-connected MNIST classifiers with 3 hidden layers each with 1024 units. We used the $L_\infty$ projection algorithm referenced in Section~\ref{sec:constrained_linear_methods}. We applied the projection to each row in the weight matrices after each gradient update. 

Our implementation of the FGS attack is standard but we found that the loss proposed by \citet{carlini2016towards} (in particular, $f_6$ which the authors found most effective) was necessary to generate attacks for the Margin-0.3 MaxMin network (and produced stronger adversarial examples for the other networks). PGD also had difficulty generating adversarial examples for the Margin-0.3 MaxMin network. It was necessary to run PGD for 200 iterations and to use a scaled down version of the random initialization typically used: instead of randomly perturbing $\vx$ in the $\epsilon$ ball we perturbed it by at most $\epsilon / 10$ before running the usual scheme. Table \ref{tab:adv_robust} summarizes our results. 

For the intepretable gradients in Figure~\ref{fig:robust_grads} we used the same architecture, but switched to 2-norm constraints. We chose a random image from classes 1-4 and computed the input-output gradient with respect to the loss function. We found that similar results were achieved with $\infty$-norm projections (and hinge loss) but the uniform gradient scale made the 2-norm-constrained input-output gradients easier to visualize.

\begin{table}[ht!]
    \centering
    \resizebox{\linewidth}{!}{
    \begin{tabular}{|l|r|r|r|r|r|}
    \hline
    \multirow{2}{*}{Model} & Clean & \multicolumn{2}{|c|}{FGS} & \multicolumn{2}{|c|}{PGD} \\ \cline{2-6}
     & \makebox Err.  $\backslash \epsilon$  & $0.1$ & $0.3$ & $0.1$ & $0.3$ \\ \hline
    Standard ReLU & 1.6 & 98.3 & 100.0 & 100.0 & 100.0 \\
    Standard MaxMin & 1.5 & 98.2 & 100.0 & 100.0 & 100.0 \\
    Margin-0.1 ReLU & 6.2 & 88.3 & 100.0 & 89.7 & 100.0 \\ 
    Margin-0.1 MaxMin & 1.9 & 36.3 & 99.2 & 44.4 & 99.8 \\
    Margin-0.3 ReLU & 16.9 & 70.1 & 100.0 & 70.3 & 100.0 \\
    Margin-0.3 MaxMin & 5.3 & 20.5 & 62.2 & 24.4 & \textbf{77.7} \\ 
    PGD 0.1 & \textbf{1.02} & 8.6 & 74.4 & 17.9 & 100.0 \\
    PGD 0.15 & 1.36 & \textbf{8.1} & \textbf{52.9} & \textbf{15.1} & 99.7 \\ \hline
    
    \end{tabular}
    }
    \caption{\textbf{Adversarial robustness} The classification error for varying $L_\infty$ distance of adversarial attacks. A perturbation size of 0.1 and 0.3 was used. }
    \label{tab:adv_robust}
\end{table}%

\end{appendices}

\end{document}